\documentclass[10pt,a4paper,notitlepage]{article}

\usepackage{moreverb,url}

\usepackage[colorlinks,bookmarksopen,bookmarksnumbered,citecolor=red,urlcolor=red]{hyperref}

\usepackage[in]{fullpage}

\usepackage[numbers]{natbib}

\usepackage{amsthm}
\usepackage{amsmath} 
\usepackage{amssymb}  
\usepackage{enumerate}
\usepackage{tikz}
\usepackage{algorithm}
\usepackage[noend]{algpseudocode}
\allowdisplaybreaks

\theoremstyle{plain}
\newtheorem{theorem}{Theorem}
\newtheorem{lemma}{Lemma}

\newcommand{\thmlab}[1]{\label{thm:#1}}
\newcommand{\lemlab}[1]{\label{lem:#1}}
\newcommand{\seclab}[1]{\label{sec:#1}}
\newcommand{\tablab}[1]{\label{tab:#1}}

\newcommand{\thmref}[1]{Theorem~\ref{thm:#1}}
\newcommand{\lemref}[1]{Lemma~\ref{lem:#1}}

\newcommand{\figref}[1]{Fig.~\ref{fig:#1}}
\newcommand{\tabref}[1]{Table~\ref{tab:#1}}

\renewcommand{\succ}{\mathrm{succ}}
\newcommand{\pred}{\mathrm{pred}}
\newcommand{\lc}{\mathrm{lc}}
\newcommand{\flc}{\mathrm{flc}}
\newcommand{\cell}{x}
\newcommand{\cellB}{y}
\newcommand{\cellC}{z}
\newcommand{\cellset}{\mathcal{X}}
\newcommand{\edgeset}{\mathcal{E}}
\newcommand{\action}{\mathcal{A}}

\DeclareMathOperator*{\argmin}{argmin}
\newcommand{\card}[1]{\left\vert #1 \right\vert}

\usetikzlibrary{arrows,shapes,positioning}
\usetikzlibrary{decorations.markings}
\tikzstyle arrowstyle=[scale=2]
\tikzstyle edgestyle=[scale=1]
\tikzstyle directed=[postaction={decorate,decoration={markings,
    mark=at position 1 with {\arrow[arrowstyle]{latex}}}}]
\tikzstyle edge=[postaction={decorate,decoration={markings,
     mark=at position 1 with {\arrow[edgestyle]{triangle 60}}}}]
\tikzstyle edge2=[postaction={decorate,decoration={markings,
     mark=at position 0.65 with {\arrow[edgestyle]{latex}}}}]

\newcommand\BibTeX{{\rmfamily B\kern-.05em \textsc{i\kern-.025em b}\kern-.08em
T\kern-.1667em\lower.7ex\hbox{E}\kern-.125emX}}

\title{Lane-Level Route Planning for Autonomous Vehicles}

\author{Mitchell Jones\footnote{Corresponding author, \texttt{mijones@nuro.ai}}\and Maximilian Haas-Heger\and Jur van den Berg}

\date{%
    Nuro\\[2ex]%
    \today
}

\begin{document}

\maketitle

\begin{abstract}
  We present an algorithm that, given a representation of a road network in lane-level detail, computes a route that minimizes the expected cost to reach a given destination. In doing so, our algorithm allows us to solve for the complex trade-offs encountered when trying to decide not just which roads to follow, but also when to change between the lanes making up these roads, in order to---for example---reduce the likelihood of missing a left exit while not unnecessarily driving in the leftmost lane. This routing problem can naturally be formulated as a Markov Decision Process (MDP), in which lane change actions have stochastic outcomes. However, MDPs are known to be time-consuming to solve in general. In this paper, we show that---under reasonable assumptions---we can use a Dijkstra-like approach to solve this stochastic problem, and benefit from its efficient $O(n \log n)$ running time. This enables an autonomous vehicle to exhibit lane-selection behavior as it efficiently plans an optimal route to its destination.\footnote{The contents of this paper are covered by US Patent 11,199,841 \citep{van-msdrpav-21}.}
\end{abstract}

\section{Introduction}
Consider the scenario in which an autonomous vehicle traversing a multi-lane road network must reach a given destination via a series of lane changes. Such a scenario occurs when the vehicle prefers to stay in the rightmost lane, but must sometimes either take an exit on the left, or make a left turn in an adjacent lane. As lane changes are not always guaranteed to succeed at a given moment due to external factors (i.e. traffic in the target lane) it is important to start attempting to make the lane change neither too early nor too late. This is just one example of the subtle choices that need to be made when determining which roads to follow, and how to navigate the lanes making up those roads.

The various routes an autonomous vehicle may take on the multi-lane road network are computed by a component of the autonomy system typically called the \emph{routing module}. The routing module will take as input an offline map (which is precomputed and contains information about lanes, road boundaries, traffic controls, and more), and a destination within the offline map defined by the user. The router is responsible for computing a route from the current position of the autonomous vehicle to its destination. This route is then passed as input to the rest of the autonomous vehicle's decision making system, which is responsible for computing the current driving behavior and translating that desired behavior into a trajectory, which can be executed in real time (see e.g. \cite{bgc+-sdc-21,paden2016survey}).

Several previous works on route planning for (autonomous) vehicles model the problem as a shortest path search on a graph with deterministic edge weights, and focus on how to handle the potentially very large graph in a memory- and compute-constrained system \citep{bdg+-rptn-16,Delling2009}. Another set of articles focus on finding an optimal route in a network with \emph{stochastic} edge weights, where the edge weights model the uncertain travel time of road segments due to traffic and congestion conditions \citep{aa-spfuc-17,Fan2005,rossi2018,9564444}. Some of these works consider the routing problem up to lane-level detail \citep{RideOS,JIANG2019305}. While these works are complementary to ours, we focus on a distinct challenge in vehicle routing: instead of stochastic edge weights, we consider stochastic outcomes of (lane change) actions, as the autonomous vehicle's motion planning module that attempts to execute the route may not be able to change lanes at any given moment due to other traffic present. This allows us to determine when the vehicle should start to attempt changing lanes in order to reduce the risk of missing turns and exits while not unduly impeding traffic in faster lanes (this is of particular concern when the autonomous vehicle is a class-A truck or a low-speed vehicle, as is the case for the authors). The result is an efficient algorithm that is able to determine the optimal route through a road network with lane change actions.

The main contributions of this paper are as follows. Firstly, we formulate our lane-level routing problem as a Markov Decision Process (MDP). The states correspond to sections of lanes in the road network, and the actions model how one navigates on these lanes; either by staying in the current lane, or deciding to try a lane change action which may or may not succeed. Secondly, we present an efficient algorithm to solve the MDP, where the running time is near-linear in the number of states of the MDP. While the optimal policy for MDPs with stochastic actions (such as the lane change action) cannot be computed efficiently in general, the key result of this paper is that under reasonable conditions we can compute the optimal policy for our routing problem using a Dijkstra-like algorithm. More specifically, if the cost formulation satisfies a \emph{monotonicity} requirement ---analogous to the nonnegative edge weight requirement in a deterministic graph search problem--- we can use Dijkstra's algorithm to find the optimal policy in a single pass. We prove that for our problem reasonable conditions imply this monotonicity, and derive them constructively.


There is a large body of work on heuristic methods to efficiently solve MDPs in practice (see e.g. \cite{afp-gps-97,bonet2003labeled,hansen2001lao,l-pa-06}). For a subset of MDPs that have a single target state out of which one cannot transition (this includes ours), the problem is sometimes referred to as the \emph{stochastic shortest path} (SSP) problem \citep{bertsekas1991analysis,gs-sspp-20}.  It is in general not possible to apply Dijkstra's algorithm to SSPs, with the exception of some special cases \citep{mg-fepmdp-05,t-eagot-95}. An implicit sufficient condition has been established previously for Dijkstra's applicability \citep{bertsekas2012dynamic} (effectively the monotonicity requirement mentioned above), and in some cases explicit sufficient conditions can also be formulated \citep{v-lsmmssppg-08}.
In this paper we do not study general MDPs or SSPs, but rather a special class for the particularly relevant application of autonomous vehicle routing. We constructively derive a reasonable explicit condition such that the implicit sufficient conditions are satisfied and Dijkstra's algorithm can in fact be applied.

The remainder of this paper is organized as follows. We begin by describing  the lane graph, the stochastic model for lane changes, and the Markov Decision Process that defines our problem. Next, we show that our MDP fulfills a monotonicity condition (\thmref{mdp-mono}), which allows for the use of an efficient Dijkstra-like algorithm. Finally, we conduct experiments which illustrate the spectrum of routing policies one can obtain in both representative and real-world environments, and compare our approach against existing MDP solvers.

\section{Preliminaries}
\seclab{definitions}
In order to define the state and action space of the MDP for our application of interest---routing an autonomous vehicle on multi-lane roads---we first define the \emph{lane graph} and present a stochastic model for changing lanes within it.

\subsection{The Lane Graph Representation}
The lane graph is a directed graph $\mathcal{G} = (\cellset, \edgeset)$, where the vertices correspond to a set of \emph{cells} $\cellset$, with each cell representing a portion of a lane that can be driven by an autonomous vehicle, and $\edgeset$ is the set of directed edges representing specific relationships between cells. There are four types of relationships between cells: left neighbor, right neighbor, successor, and predecessor. We construct the lane graph such that cells have one-to-one neighbor relationships, resulting in lane graphs with a rectangular pattern like the ones shown in \figref{cells}. More formally, the lane graph has the following properties:
\begin{itemize}
\item Each cell $\cell \in \cellset$ has at most one left neighbor, denoted $\mathrm{lnb}(\cell)$, and at most one right neighbor, denoted $\mathrm{rnb}(\cell)$. If a cell $\cell_2$ is a neighbor of cell $\cell_1$, then $\cell_2$ is accessible from $\cell_1$ through a lane change along the entire extent of cell $\cell_1$. The neighbor relationship is symmetric: $\cell_1 = \mathrm{lnb}(\cell_2) \Longleftrightarrow \cell_2 = \mathrm{rnb}(\cell_1)$.
\item Each cell $\cell$ has a set $\succ(\cell)$ of successors and a set $\pred(\cell)$ of predecessors. If cell $\cell_2$ is a successor of cell $\cell_1$, then $\cell_2$ is accessible from $\cell_1$ by continuing to drive in the same lane. The successor and predecessor relationships are symmetric: $\cell_1 \in \succ(\cell_2) \Longleftrightarrow \cell_2 \in \pred(\cell_1)$.
\item Each cell $\cell$ has an associated \emph{length} $\ell(\cell) > 0$.
\item Each cell $\cell$ has an associated \emph{cost} $c(\cell) > 0$ of traversing the cell.
\end{itemize}
For our analysis we assume that every cell has on average a constant number of successors and predecessors (in addition to at most two neighbors). Note that in most lane graphs the majority of lane cells will have just one successor and predecessor, unless a lane forks or merges, in which case it will have more than one successor or predecessor, respectively. This assumption implies that the graph is sparse, with $\card{\edgeset} = O(\card{\cellset})$.

There are some practical considerations to take into account when deciding on the length of the cells making up the lane graph. They must be divided such that they support the one-to-one neighbor relationships, meaning that a merge, a fork, or any other change in neighbor relations among lanes  will force cell boundaries across the width of the road. The cells should also not be too short, as that will increase the total number of cells, and therefore the computational expense in computing an optimal policy. At the same time, they should not be too long, as the cells effectively discretize the policy we are computing. The trade-off in total number of cells versus cell length will be explored in the experimental section.

For the purposes of this paper, we define that the length of neighboring cells is equal, i.e. $\cell_1 = \cell_7 = \cell_{11} = \cell_{15}$ in \figref{cells}. If in practice the lane cells are not of exactly equal length (e.g. in a curve), we can define their length to be equal to the maximum length among the lane cells in the transitive closure of the neighbor relationship.

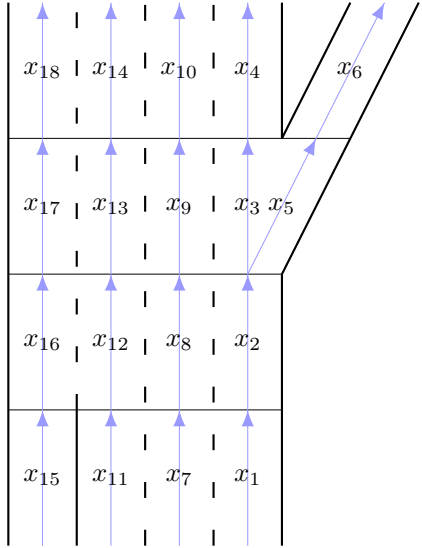
\begin{figure}[t]
\centering
\begin{tikzpicture}[scale=0.9]
\draw[thick] (0,1) -- (0,9);
\draw[thick] (4,9) -- (4,7);
\draw[thick] (4,5) -- (4,1);

\draw[thick] (4,5) -- (6,9);
\draw[thick] (4,7) -- (5,9);

\draw[ultra thin] (0,3) -- (4,3);
\draw[ultra thin] (0,5) -- (4,5);
\draw[ultra thin] (0,7) -- (5,7);

\draw[thick] (1, 1) -- (1, 3);
\draw[thick, dash pattern=on 6pt off 12pt] (1, 3) -- (1, 9);
\draw[thick, dash pattern=on 6pt off 12pt] (2, 1) -- (2, 9);
\draw[thick, dash pattern=on 6pt off 12pt] (3, 1) -- (3, 9);

\draw[ultra thin, blue!40!white, directed] (0.5, 1) -- (0.5, 3) node[pos=0.5, black] {$\cell_{15}$};
\draw[ultra thin, blue!40!white, directed] (0.5, 3) -- (0.5, 5) node[pos=0.5, black] {$\cell_{16}$};
\draw[ultra thin, blue!40!white, directed] (0.5, 5) -- (0.5, 7) node[pos=0.5, black] {$\cell_{17}$};
\draw[ultra thin, blue!40!white, directed] (0.5, 7) -- (0.5, 9) node[pos=0.5, black] {$\cell_{18}$};

\draw[ultra thin, blue!40!white, directed] (1.5, 1) -- (1.5, 3) node[pos=0.5, black] {$\cell_{11}$};
\draw[ultra thin, blue!40!white, directed] (1.5, 3) -- (1.5, 5) node[pos=0.5, black] {$\cell_{12}$};
\draw[ultra thin, blue!40!white, directed] (1.5, 5) -- (1.5, 7) node[pos=0.5, black] {$\cell_{13}$};
\draw[ultra thin, blue!40!white, directed] (1.5, 7) -- (1.5, 9) node[pos=0.5, black] {$\cell_{14}$};

\draw[ultra thin, blue!40!white, directed] (2.5, 1) -- (2.5, 3) node[pos=0.5, black] {$\cell_7$};
\draw[ultra thin, blue!40!white, directed] (2.5, 3) -- (2.5, 5) node[pos=0.5, black] {$\cell_8$};
\draw[ultra thin, blue!40!white, directed] (2.5, 5) -- (2.5, 7) node[pos=0.5, black] {$\cell_9$};
\draw[ultra thin, blue!40!white, directed] (2.5, 7) -- (2.5, 9) node[pos=0.5, black] {$\cell_{10}$};

\draw[ultra thin, blue!40!white, directed] (3.5, 1) -- (3.5, 3) node[pos=0.5, black] {$\cell_1$};
\draw[ultra thin, blue!40!white, directed] (3.5, 3) -- (3.5, 5) node[pos=0.5, black] {$\cell_2$};
\draw[ultra thin, blue!40!white, directed] (3.5, 5) -- (3.5, 7) node[pos=0.5, black] {$\cell_3$};
\draw[ultra thin, blue!40!white, directed] (3.5, 7) -- (3.5, 9) node[pos=0.5, black] {$\cell_4$};

\draw[ultra thin, blue!40!white, directed] (3.5, 5) -- (4.5, 7) node[pos=0.5, black] {$\cell_5$};
\draw[ultra thin, blue!40!white, directed] (4.5, 7) -- (5.5, 9) node[pos=0.5, black] {$\cell_6$};
\end{tikzpicture}
\caption{An example of a lane graph partitioning a road map into cells such that it has the desired properties. The centerlines of the lanes are shown in light blue. Because of the requirement that neighboring relationships are consistent and one-to-one, the cells $\{\cell_2, \cell_8, \cell_{12}, \cell_{16}\}$ are distinct from the cells $\{\cell_1, \cell_7, \cell_{11}, \cell_{15}\}$ because the neighboring relationship among the two left lanes change (note that for the purpose of this paper, we do not consider a pair of lane cells separated by a solid lane boundary as neighbors, e.g.\ $\cell_{11}$ and $\cell_{15}$, as a lane change between them is illegal).
The separation between cells $\cell_{15}$ and $\cell_{16}$, and between $\cell_{11}$ and $\cell_{12}$ in the left two lanes ``cascades'' to the right lanes to maintain the one-to-one neighbor relationships.
Similarly, the cells $\{\cell_3, \cell_9, \cell_{13}, \cell_{17}\}$ are distinct from the cells $\{\cell_2, \cell_8, \cell_{12}, \cell_{16}\}$ because the right lane splits into two. Finally, the cells $\{\cell_4, \cell_{10}, \cell_{14}, \cell_{18}\}$ are separate from $\{\cell_3, \cell_9, \cell_{13}, \cell_{17}\}$ because the right lane is no longer a neighbor of the exiting lane.}
\label{fig:cells}
\end{figure}

\subsection{A Stochastic Model for Lane Changes}
It is possible to change lanes from cell $\cell_1$ to cell $\cell_2$ if $\cell_2$ is a neighbor of $\cell_1$ in the lane graph.
A lane change may in practice not always succeed however, because at any given time traffic in neighboring cells may make a lane change impossible. We therefore define a stochastic model for lane changes. Obviously, a lane change between two longer cells has a higher probability of eventually succeeding than a lane change between shorter cells.

Let $f : \mathbb{R}^+ \to [0, 1]$ define the probability that a lane change between two lane graph cells of length $\ell$ succeeds. This probability function should have the following properties:
\begin{align}
\begin{split}
& f(0) = 0,\\
& f(\infty) = 1,\\
& f(\ell_1 + \ell_2) = f(\ell_1) + (1 - f(\ell_1))f(\ell_2). \label{eq:psum}
\end{split}
\end{align}
That is, a lane change between two cells of zero length will never succeed, and a lane change between two cells of infinite length will surely succeed. The last property of Eq.\ \eqref{eq:psum} guarantees that the probability function is invariant to partitioning lanes into multiple cells of different length.

It can easily be verified that
\begin{align}
f(\ell) &= 1 - \exp(-\alpha \ell), & \alpha &> 0, \label{eq:pl}
\end{align}
is the unique solution satisfying these properties, and one readily recognizes in $f(\ell)$ the cumulative density function of an exponential distribution with rate parameter $\alpha$. In our case, $\alpha$ can be interpreted as the average number of successful lane changes per unit of length if one were to constantly try to change lanes.


\subsection{The Markov Decision Process}
\seclab{mdp}
We formulate our problem as a Markov Decision Process (MDP): The set of states are exactly the set of cells $\cellset$, the set of actions $\action(\cell)$ are the actions one can take from each cell $\cell \in \cellset$ (to be defined shortly), the cost function $c(\cell, a, \cell')$ defines the cost of moving from $\cell \in \cellset$ to $\cell' \in \cellset$ under action $a \in \action(\cell)$, and $p(\cell' | \cell, a)$ is a transition probability function defining the probability one arrives at $\cell' \in \cellset$ when taking action $a \in \action(\cell)$ from cell $\cell \in \cellset$.

Given a goal cell $\cell_g \in \cellset$, the objective is to compute for each cell $\cell \in \cellset$ the optimal \emph{expected} cost $g(\cell)$ to reach the goal from $\cell$ when taking optimal actions (referred to in this paper as the \emph{value} or \emph{cost-to-go}), and to compute the optimal action $\pi(\cell)$ one should take for each cell. By definition, $g(\cell_g) = 0$ and cells $\cell$ from which the goal cannot be reached have $g(\cell) = \infty$.

The solution is generally defined by the Bellman equation:
\begin{align}
q(\cell, a) & = \sum_{\cell'\in\cellset} p(\cell' | \cell,a)(c(\cell, a, \cell') + g(\cell')), \label{eq:q} \\
    g(\cell) & = \min_{a\in\action(\cell)} q(\cell, a), ~~~~~~~ g(\cell_g) = 0, \label{eq:g} \\
    \pi(\cell) & = \argmin_{a\in\action(\cell)} q(\cell, a), ~~~~~~ \pi(\cell_g) = \varnothing. \label{eq:pi}
\end{align}
where $q(\cell, a)$ is the expected cost to reach the goal when taking action $a$ from $\cell$ and optimal actions thereafter. Let us precisely define the set of actions $\action(\cell)$ one can take from a cell $\cell$, and their associated costs and transition probabilities.

\begin{figure}[t]
\centering
\begin{tikzpicture}
\draw[thick] (0,0.5) -- (0,5.2);
\draw[thick] (2,0.5) -- (2,5.2);
\draw[thick, dash pattern=on 6pt off 12pt] (1, 0.5) -- (1, 5.2) node[pos=0.0, below, black] {\footnotesize{(a)}};

\draw[ultra thin] (0,4) -- (2,4);
\draw[ultra thin] (0,1) -- (2,1);

\draw[ultra thick, red, edge] (0.5, 1.0) -- (0.5, 4) node[pos=0.5, left, black] {$\cell$};
\draw[ultra thin, blue!40!white, directed] (0.5, 4) -- (0.5, 5.2) node[pos=0.5, black] {$\cell_s$};
\draw[ultra thin, blue!40!white, directed] (1.5, 1) -- (1.5, 4) node[pos=0.5, black] {$\cell_n$};
\draw[ultra thin, blue!40!white, directed] (1.5, 4) -- (1.5, 5.2) node[pos=0.5, black] {$\cell_{ns}$};
\draw[ultra thin, blue!40!white, directed] (0.5, 0.5) -- (0.5, 1);
\draw[ultra thin, blue!40!white, directed] (1.5, 0.5) -- (1.5, 1);

\draw[thick] (3,0.5) -- (3,5.2);
\draw[thick] (5,0.5) -- (5,5.2);
\draw[thick, dash pattern=on 6pt off 12pt] (4, 0.5) -- (4, 5.2) node[pos=0.0, below, black] {\footnotesize{(b)}};

\draw[ultra thin] (3,4) -- (5,4);
\draw[ultra thin] (3,1) -- (5,1);

\draw[ultra thin, blue!40!white, directed] (3.5, 4) -- (3.5, 5.2) node[pos=0.5, black] {$\cell_s$};
\draw[ultra thin, blue!40!white] (4.5, 1) -- (4.5, 4) node[pos=0.5, black] {$\cell_n$};
\draw[ultra thin, blue!40!white, directed] (4.5, 4) -- (4.5, 5.2) node[pos=0.5, black] {$\cell_{ns}$};
\draw[ultra thin, blue!40!white, directed] (3.5, 0.5) -- (3.5, 1);
\draw[ultra thin, blue!40!white, directed] (4.5, 0.5) -- (4.5, 1);

\draw[ultra thick, dash pattern=on 10pt off 8pt, red, edge] (3.5, 1) -- (3.5, 4) node[pos=0.5, left, black] {$\cell$};
\draw[ultra thick, red, edge] (3.5, 1) -- (4.5, 4);

\draw[thick] (6,0.5) -- (6,5.2);
\draw[thick] (8,0.5) -- (8,5.2);
\draw[thick, dash pattern=on 6pt off 12pt] (7, 0.5) -- (7, 5.2) node[pos=0.0, below, black] {\footnotesize{(c)}};

\draw[ultra thin] (6,4) -- (8,4);
\draw[ultra thin] (6,1) -- (8,1);

\draw[ultra thin, blue!40!white, directed] (6.5, 1) -- (6.5, 4) node[pos=0.5, black] {$\cell$};
\draw[ultra thin, blue!40!white, directed] (6.5, 4) -- (6.5, 5.2) node[pos=0.5, black] {$\cell_s$};
\draw[ultra thin, blue!40!white] (7.5, 1) -- (7.5, 4) node[pos=0.5, black] {$\cell_n$};
\draw[ultra thin, blue!40!white, directed] (7.5, 4) -- (7.5, 5.2) node[pos=0.5, black] {$\cell_{ns}$};
\draw[ultra thin, blue!40!white, directed] (6.5, 0.5) -- (6.5, 1);
\draw[ultra thin, blue!40!white, directed] (7.5, 0.5) -- (7.5, 1);

\draw[ultra thick, red, edge] (6.5, 1) -- (7.5, 4);
\end{tikzpicture}
\vspace{-5pt}
\caption{Notation of possible actions. (a) a stay-in-lane action $a_s$, (b) a lane change action $a_\lc$, and (c) a forced lane change action $a_\flc$. We visualize the policy by drawing arrows from the beginning of a cell to the beginning of the target cell under the policy. This arrow, however, only denotes the action we want the motion planning framework to attempt while traversing the cell. It does not indicate for example the trajectory we expect the robot to take in changing lanes. The actual lane change maneuver may be completed within a fraction of a cell, or over the space of several successive cells, depending on the length of the cells and external factors such as traffic.}
\label{fig:actions}
\vspace{-10pt}
\end{figure}
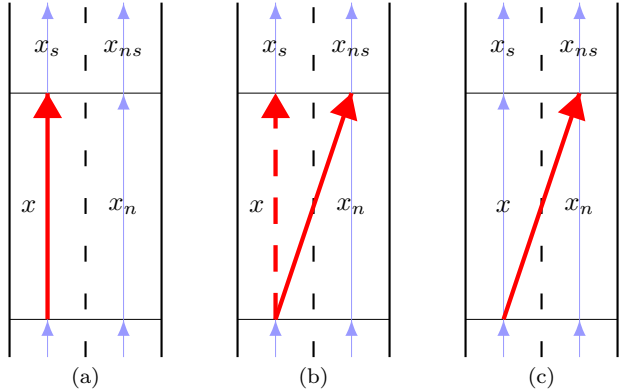

\textbf{Stay-in-lane actions:\,}
    For each successor $\cell_{s} \in \succ(\cell)$ of $\cell$, we have a stay-in-lane action $a_{s}$ that would route the vehicle from cell $\cell$ to cell $\cell_{s}$ (see \figref{actions}(a)). We have the following cost and transition probabilities for action $a_s$:
    \begin{align}
        c(\cell, a_s, \cell_s) & = c(\cell), &
        p(\cell_s | \cell, a_s) &= 1, \label{eq:as}
    \end{align}
    where $c(\cell)$ is the cost to traverse cell $\cell$.

\textbf{Lane change actions:\,}
For each neighbor $\cell_n \in \mathrm{lnb}(\cell) \cup \mathrm{rnb}(\cell)$ of $\cell$, we have a lane change action $a_\lc$ for each pair of successors $(\cell_{s}, \cell_{ns}) \in \succ(\cell) \times \succ(\cell_{n})$, which would route the vehicle from cell $\cell$ to cell $\cell_{ns}$ if the lane change is successful and to $\cell_s$ if the lane change is unsuccessful (see \figref{actions}(b)). We have the following costs and transition probabilities for action $a_\lc$:
    \begin{align}
      & c(\cell, a_\lc, \cell_{ns}) = c_\lc + c(\cell),
      p(\cell_{ns} | \cell, a_\lc) = f(\ell(\cell)), \label{eq:alc1} \\
      & c(\cell, a_\lc, \cell_s) = c(\cell),
      p(\cell_{s} | \cell, a_\lc) = 1 - f(\ell(\cell)), \label{eq:alc2}
    \end{align}
where $\ell(x)$ is the length of cell $\cell$, $f(\cdot)$ is as in Eq.\ \eqref{eq:pl}, and $c_\lc \geq 0$ is the cost for making a lane change. This cost is applied in order to discourage the vehicle from changing lanes unless necessary.

\textbf{Forced lane change actions:\,}
Cells from which the goal cannot be reached will have a cost-to-go of $\infty$. Since our model of lane changes is stochastic, large parts of our lane graph will have a nonzero probability, however small, of arriving in one of these cells, and would therefore have an expected cost-to-go of infinity as well. This would render the formulation of our problem useless. To avoid this, we define an additional action we call a ``forced'' lane change, which is guaranteed to succeed, but comes at a large but finite additional cost.

For each neighbor $\cell_n \in \mathrm{lnb}(\cell) \cup \mathrm{rnb}(\cell)$ of $\cell$, we have a forced lane change action $a_\flc$ for each successor $\cell_{ns} \in \succ(\cell_n)$, which would route the vehicle from cell $\cell$ to cell $\cell_{ns}$ regardless of whether the lane change would normally succeed  (see \figref{actions}(c)). However, a (large) extra cost of $c_\flc \geq 0$ is applied if the lane change would normally not succeed. We have the following cost and transition probabilities for action $a_\flc$:
    \begin{align}
    \begin{split}
        c(\cell, a_\flc, \cell_{ns}) & = c_\lc + c(\cell) + (1 - f(\ell(x)))c_\flc, \\
        p(\cell_{ns} | \cell, a_\flc) &= 1. \label{eq:aflc}
    \end{split}
    \end{align}



\section{Computing the Optimal Policy}
\seclab{solution}
The Markov Decision Process as defined above, can be solved in its general form using the standard value iteration or policy iteration algorithms. Both approaches, however, suffer from a running time that is at least quadratic in the number of states \citep{l-pa-06}. Fortunately, under reasonable conditions, we can solve the problem using a Dijkstra-like approach that is non-iterative and runs in $O(n \log n)$ time, where $n = \card{\cellset}$.



\subsection{Monotonicity Conditions}
\seclab{monotonicity}
For a Dijkstra-like approach to work, the cost formulation must be \emph{monotone} \citep{bertsekas2012dynamic,l-pa-06}. That is, the cost-to-go $g(\cell)$ of each cell $\cell$ must be larger than the cost-to-go $g(\cell')$ of the cells $\cell'$ one could arrive at from $\cell$ when taking the optimal action $\pi(\cell)$. The monotonicity requirement is formally stated as follows:
\begin{align}
   \forall\{\cell, \cell' \, |\, p(\cell' | \cell, \pi(\cell)) > 0\} : g(\cell) > g(\cell'). \label{eq:monotonicity}
\end{align}
This is the equivalent of the requirement of nonnegative edge costs in Dijkstra's algorithm. It guarantees that the optimal value of $\cell$ depends only on cells that have been previously visited in Dijkstra's algorithm (we remark that the monotonicity condition requires a strict inequality to avoid cyclic dependencies \citep{v-lsmmssppg-08}). Therefore, a single pass of the algorithm suffices to find the optimal policy.

Let us consider under what conditions our cost formulation is monotone. For stay-in-lane actions and forced lane change actions, the monotonicity requirement trivially holds. If stay-in-lane action $a_s$ is optimal, we have $\cell_s$, a successor of $\cell$ as the only cell one can arrive at. And trivially:
\begin{align}
    g(\cell) = q(\cell, a_s) = c(\cell) + g(\cell_s) > g(\cell_s).
    \label{eq:sil}
\end{align}

If a forced lane change action $a_\flc$ is optimal, we have $\cell_{ns}$, the successor of a neighbor of $\cell$ as the only cell one can arrive at. And trivially:
\begin{align}
    g(\cell) = q(\cell, a_\flc) = c_\lc + c(\cell) + (1 - f(\ell(x)))c_\flc + g(\cell_{ns}) > g(\cell_{ns}).
    \label{eq:flc}
\end{align}

For regular lane change actions we have two potential next cells: $\cell_s$, a successor of $\cell$, and $\cell_{ns}$, a successor of a neighbor of $\cell$. We show that the monotonicity requirement for regular lane change actions holds under the following condition:

\begin{lemma}
\lemlab{lane-change-mono}
If $\forall \cell \in \cellset : c(\cell)/\ell(\cell) \geq \alpha c_\flc$, where $\alpha$ is the lane change success rate, then for any cell $\cell$ from which a lane change action $a_{lc}$ is optimal, we have $g(\cell) > g(\cell_s) \wedge g(\cell) > g(\cell_{ns})$.
\end{lemma}

The above result gives a very reasonable condition for monotonicity. For instance, setting $c_\flc = 1/\alpha$ and requiring $c(\cell) \geq \ell(\cell)$ for all cells $\cell$ would satisfy the monotonicity requirement, and allows us to use a Dijkstra-like approach. The above Lemma is proved in subsequent \lemref{lane-change-good} and \lemref{lane-change-bad}.

\begin{lemma}
\lemlab{lane-change-good}
For any cell $\cell$ from which a lane change action $a_\lc$ is optimal, we have $g(\cell) > g(\cell_{ns})$.
\end{lemma}
\begin{proof}
As $a_\lc$ is optimal, we have $g(\cell) = q(\cell, a_\lc)$, and we must have that $q(\cell, a_\lc) \leq q(\cell, a_s)$ (a stay-in-lane action). This provides an upper bound on the value of $g(\cell_{ns})$:
\begin{align}
& ~~ q(\cell, a_\lc) \leq q(\cell, a_s) \nonumber \\
\overset{\textrm{i}}{\Longleftrightarrow} & ~~  f(\ell(x)) (c_\lc+c(\cell)+g(\cell_{ns})) \nonumber + (1-f(\ell(x))) (c(\cell) + g(\cell_s)) \leq c(\cell) + g(\cell_s) \nonumber\\
\overset{\textrm{ii}}{\Longleftrightarrow} & ~~ f(\ell(x)) (c_\lc + g(\cell_{ns})) \leq f(\ell(x))g(\cell_s) \nonumber \\
\overset{\textrm{iii}}{\Longleftrightarrow} & ~~  g(\cell_{ns}) \leq g(\cell_s) - c_\lc. \label{eq:inequality_gns}
\end{align}
Equivalence (i) follows from expanding both sides using Eq.\ \eqref{eq:q} and respectively Eqs.\ \eqref{eq:alc1}--\eqref{eq:alc2} and Eq.\ \eqref{eq:as}. Equivalence (ii) follows from rearranging and eliminating terms, and equivalence (iii) follows from dividing both sides by $f(\ell(x))$.

We complete the proof by showing that 
this bound implies $q(\cell, a_\lc) > g(\cell_{ns})$:
\begin{align*}
& ~~ q(\cell, a_\lc) > g(\cell_{ns})\\
\overset{\textrm{i}}{\Longleftrightarrow} & ~~ f(\ell(x)) (c_\lc+c(\cell)+g(\cell_{ns})) + (1-f(\ell(x))) (c(\cell) + g(\cell_s)) > g(\cell_{ns}) \\
\overset{\textrm{ii}}{\Longleftrightarrow} & ~~ c(\cell) + f(\ell(x)) c_\lc + (1 - f(\ell(x))) g(\cell_s) > (1 - f(\ell(x))) g(\cell_{ns}) \\
\overset{\textrm{iii}}{\Longleftarrow} & ~~ c(\cell) + f(\ell(x)) c_\lc + (1 - f(\ell(x))) g(\cell_s) > (1 - f(\ell(x))) (g(\cell_s) - c_\lc) \\
\overset{\textrm{iv}}{\Longleftrightarrow} & ~~ c(\cell) + c_\lc > 0.
\end{align*}
Equivalence (i) follows from expanding the left-hand side using Eq.\ \eqref{eq:q} and Eqs.\ \eqref{eq:alc1}--\eqref{eq:alc2}. Equivalence (ii) follows from rearranging terms. Implication (iii) follows from  inequality \eqref{eq:inequality_gns}. Equivalence (iv) follows from rearranging and eliminating terms. The last sum is positive as $c_\lc \geq 0$ and $c(\cell) > 0$ for all $\cell \in \cellset$. \qed
\end{proof}

\begin{lemma}
\lemlab{lane-change-bad}
If $\forall \cell \in \cellset : c(\cell)/\ell(\cell) \geq \alpha c_\flc$, then for any cell $\cell$ from which a lane change action $a_{lc}$ is optimal, we have  $g(\cell) > g(\cell_{s})$.
\end{lemma}
\begin{proof}
As $a_\lc$ is optimal, we have $g(\cell) = q(\cell, a_\lc)$, and we must have that $q(\cell, a_\lc) \leq q(\cell, a_\flc)$ (a forced lane change action). This provides an upper bound on the value of $g(\cell_{s})$:
\begin{align}
   & ~~ q(\cell, a_\lc) \leq q(\cell, a_\flc) \nonumber \\
\overset{\textrm{i}}{\Longleftrightarrow} & ~~  f(\ell(x)) (c_\lc+c(\cell)+g(\cell_{ns})) \nonumber+ (1-f(\ell(x))) (c(\cell) + g(\cell_s)) \leq c_\lc + c(\cell) + (1 - f(\ell(x))) c_\flc + g(\cell_{ns}) \nonumber \\
\overset{\textrm{ii}}{\Longleftrightarrow} & ~~ (1 - f(\ell(x))) g(\cell_s) \leq (1 - f(\ell(x))) (c_\lc + c_\flc + g(\cell_{ns})) \nonumber \\
\overset{\textrm{iii}}{\Longleftrightarrow} & ~~ g(\cell_s) \leq g(\cell_{ns}) + c_\flc + c_\lc. \label{eq:inequality_gs}
\end{align}
Equivalence (i) follows from expanding both sides using Eq.\ \eqref{eq:q} and respectively Eqs.\ \eqref{eq:alc1}--\eqref{eq:alc2} and Eq.\ \eqref{eq:aflc}. Equivalence (ii) follows from rearranging and eliminating terms, and equivalence (iii) follows from dividing both sides by $(1-f(\ell(x)))$.

We complete the proof by showing that this bound and the conditions of the lemma imply $q(\cell, a_\lc) > g(\cell_{s})$:
\begin{align*}
    & ~~ q(\cell, a_\lc) > g(\cell_s) \\
\overset{\textrm{i}}{\Longleftrightarrow} & ~~ f(\ell(x)) (c_\lc+c(\cell)+g(\cell_{ns})) + (1-f(\ell(x))) (c(\cell) + g(\cell_s)) > g(\cell_s) \\
\overset{\textrm{ii}}{\Longleftrightarrow} & ~~ c(\cell) + f(\ell(x)) (c_\lc + g(\cell_{ns})) > f(\ell(x))g(\cell_s) \\
\overset{\textrm{iii}}{\Longleftarrow} & ~~  c(\cell) + f(\ell(x)) (c_\lc + g(\cell_{ns})) > f(\ell(x))(g(\cell_{ns}) + c_\flc + c_\lc) \\
\overset{\textrm{iv}}{\Longleftrightarrow} & ~~ c(\cell) > f(\ell(x))c_\flc \\
\overset{\textrm{v}}{\Longleftarrow} & ~~ c(\cell) \geq \alpha \ell(\cell) c_\flc \\
\overset{\textrm{vi}}{\Longleftrightarrow} & ~~ c(\cell) / \ell(\cell) \geq \alpha c_\flc.
\end{align*}
Equivalence (i) follows from expanding using Eq.\ \eqref{eq:q} and Eqs.\ \eqref{eq:alc1}--\eqref{eq:alc2}. Equivalence (ii) follows from rearranging terms. Implication (iii) follows from inequality \eqref{eq:inequality_gs}. Equivalence (iv) follows from eliminating terms. Implication (v) follows as $f(\ell) = 1 - \exp(-\alpha \ell) < \alpha \ell$ for $\ell > 0$, and $\ell(x) > 0$ for all $x \in \cellset$. Through dividing by $\ell(\cell)$ (equivalence (vi)) we arrive at the condition of the lemma. \qed
\end{proof}

The proof of \lemref{lane-change-mono} now follows directly by combining \lemref{lane-change-good} and \lemref{lane-change-bad}. The discussion above culminates in the following Theorem.

\begin{theorem}
\thmlab{mdp-mono}
Let $\mathcal{G} = (\cellset, \edgeset)$ be the lane graph. If $\forall\cell \in \cellset : c(\cell)/\ell(\cell) \geq \alpha c_\flc$, then the monotonicity requirement of Eq.\ \eqref{eq:monotonicity} is satisfied:
$$\forall\{\cell, \cell' \, |\, p(\cell' | \cell, \pi(\cell)) > 0\} : g(\cell) > g(\cell').$$
\end{theorem}

Note, that monotonicity also implies that the resulting policy can be represented as a directed acyclic graph where the vertices represent cells and directed edges denote the cells one may possibly enter from any given cell under the optimal policy.

\subsection{Dijkstra-Like Algorithm}
\seclab{dijkstra}
Given the conditions for monotonicity stated in the previous section, we now give a Dijkstra-like algorithm to find the optimal costs-to-go (values) and actions for all cells in the lane graph given a goal cell $\cell_g$ (see Algorithm~\ref{alg:router}). Note that in this section and beyond, we highlight that ``successor'' and ``predecessor'' of a cell refer to the cells that can be reached by driving along a lane, as defined in the preliminaries. We define the \emph{children} of a node $\cell$ as all cells $\cell'$ that reach $\cell$ through a single action $a$ with nonzero probability, i.e., the set $\{\cell' \in \cellset : p(\cell' | \cell, a) > 0\}$.

Initially, the value of all cells is set to infinity, except for the goal cell, whose value is set to zero (line \ref{line:init}). The algorithm maintains a priority queue $\mathcal{Q}$ of all \emph{open} cells that have been given a value, but whose optimal value has yet to be confirmed. Initially, the queue only contains the goal cell (line \ref{line:initq}).

In each iteration, the cell $\bar{\cell}$ in the queue with the minimum value is taken and removed from the queue (lines \ref{line:top} and \ref{line:pop}). At this point, $\bar{\cell}$'s optimal value has been confirmed and is said to be \emph{closed}. The optimality follows from \thmref{mdp-mono}, as the cost-to-go function $g$ is monotone. We then look at all of the children of $\bar{\cell}$ (line \ref{line:children}). %
If the value of the child $\cell$ can be decreased by taking action $a$ (line \ref{line:decrease}), then we update its value (line \ref{line:update}) and set its optimal action (line \ref{line:action}). Subsequently, $\cell$ is either inserted into the priority queue with key $g(\cell)$ or if it was already in the priority queue, its key is decreased (line \ref{line:add}). It should be noted that any cell whose value is decreased must either be newly visited (i.e. its prior value was infinity), or it is open and therefore in the queue. If somehow a cell whose value is updated is neither in the queue nor had a prior value of infinity, it must have been closed before, which means that the cost formulation is not monotone. One can check for this explicitly in any implementation.

The algorithm continues until the queue is empty (see line \ref{line:while}) and the optimal value of all cells has been confirmed. The stored optimal actions $\pi(\cell)$ can now be used to execute the route.

\begin{algorithm}[t]
\caption{\textsc{Router}$(\cell_g)$}\label{alg:router}
\begin{algorithmic}[1]
\State $\forall \cell : \pi(\cell) \gets \varnothing$; $\forall \cell : g(\cell) \gets \infty$; $g(\cell_g) \gets 0$ \label{line:init}
\State $\mathcal{Q} \gets \{\cell_g\}$ \label{line:initq}
\While{$\mathcal{Q} \neq \varnothing$} \label{line:while}
  \State $\bar{\cell} \gets \argmin_{\cell \in \mathcal{Q}} g(\cell)$ \label{line:top}
  \State $\mathcal{Q} \gets \mathcal{Q} \setminus \{\bar{\cell}\}$ \label{line:pop}
  \ForAll{$(\cell, a) \in \{ (\cell, a) \, | \, p(\bar{\cell}|\cell, a) > 0\}$} \label{line:children}
  \State $q(\cell, a) \gets \sum_{\cell'} p(\cell' | \cell,a)(c(\cell, a, \cell')+g(\cell'))$ \label{line:qvalues}
  \If{$q(\cell, a) < g(\cell)$} \label{line:decrease}
    \State $g(\cell) \gets q(\cell, a)$ \label{line:update}
    \State $\pi(\cell) \gets a$ \label{line:action}
    \State $\mathcal{Q} \gets \mathcal{Q} \cup \{\cell\}$ \label{line:add}
  \EndIf
  \EndFor
\EndWhile
\Return{$g(\cdot), \pi(\cdot)$}
\end{algorithmic}
\end{algorithm}

\begin{figure*}
\centering
\begin{tikzpicture}[scale=0.96]
\draw[thick] (0,0.5) -- (0,5.2);
\draw[thick] (2,0.5) -- (2,5.2);
\draw[thick, dash pattern=on 6pt off 12pt] (1, 0.5) -- (1, 5.2) node[pos=0.0, below, black] {\footnotesize{(a)}};

\draw[ultra thin] (0,4) -- (2,4);
\draw[ultra thin] (0,1) -- (2,1);

\draw[ultra thick, red, edge] (0.5, 1.0) -- (0.5, 4) node[pos=0.5, left, black] {$\cell_p\!\!$};
\draw[ultra thin, blue!40!white, directed] (0.5, 4) -- (0.5, 5.2) node[pos=0.5, black] {$\cell$};
\draw[ultra thin, blue!40!white, directed] (1.5, 1) -- (1.5, 4) node[pos=0.5, black] {$\cell_n$};
\draw[ultra thin, blue!40!white, directed] (1.5, 4) -- (1.5, 5.2) node[pos=0.5, black] {$\cell_s$};
\draw[ultra thin, blue!40!white, directed] (0.5, 0.5) -- (0.5, 1);
\draw[ultra thin, blue!40!white, directed] (1.5, 0.5) -- (1.5, 1);

\draw[thick] (2.5,0.5) -- (2.5,5.2);
\draw[thick] (4.5,0.5) -- (4.5,5.2);
\draw[thick, dash pattern=on 6pt off 12pt] (3.5, 0.5) -- (3.5, 5.2) node[pos=0.0, below, black] {\footnotesize{(b)}};

\draw[ultra thin] (2.5,4) -- (4.5,4);
\draw[ultra thin] (2.5,1) -- (4.5,1);

\draw[ultra thin, blue!40!white] (3, 1) -- (3, 4) node[pos=0.5, black] {$\cell_p$};
\draw[ultra thin, blue!40!white, directed] (3, 4) -- (3, 5.2) node[pos=0.5, black] {$\cell$};
\draw[ultra thin, blue!40!white, directed] (4, 1) -- (4, 4) node[pos=0.5, black] {$\cell_n$};
\draw[ultra thin, blue!40!white, directed] (4, 4) -- (4, 5.2) node[pos=0.5, black] {$\cell_s$};
\draw[ultra thin, blue!40!white, directed] (3, 0.5) -- (3, 1);
\draw[ultra thin, blue!40!white, directed] (4, 0.5) -- (4, 1);

\draw[ultra thick, red, edge] (4, 1) -- (3, 4);

\draw[thick] (5,0.5) -- (5,5.2);
\draw[thick] (7,0.5) -- (7,5.2);
\draw[thick, dash pattern=on 6pt off 12pt] (6, 0.5) -- (6, 5.2) node[pos=0.0, below, black] {\footnotesize{(c)}};

\draw[ultra thin] (5,4) -- (7,4);
\draw[ultra thin] (5,1) -- (7,1);

\draw[ultra thin, blue!40!white, directed] (5.5, 4) -- (5.5, 5.2) node[pos=0.5, black] {$\cell$};
\draw[ultra thin, blue!40!white] (6.5, 1) -- (6.5, 4) node[pos=0.5, black] {$\cell_n$};
\draw[ultra thin, blue!40!white, directed] (6.5, 4) -- (6.5, 5.2) node[pos=0.5, black] {$\cell_s$};
\draw[ultra thin, blue!40!white, directed] (5.5, 0.5) -- (5.5, 1);
\draw[ultra thin, blue!40!white, directed] (6.5, 0.5) -- (6.5, 1);

\draw[ultra thick, dash pattern=on 10pt off 8pt, red, edge] (5.5, 1) -- (5.5, 4) node[pos=0.5, left, black] {$\cell_p\!\!$};
\draw[ultra thick, red, edge] (5.5, 1) -- (6.5, 4);

\draw[thick] (7.5,0.5) -- (7.5,5.2);
\draw[thick] (9.5,0.5) -- (9.5,5.2);
\draw[thick, dash pattern=on 6pt off 12pt] (8.5, 0.5) -- (8.5, 5.2) node[pos=0.0, below, black] {\footnotesize{(d)}};

\draw[ultra thin] (7.5,4) -- (9.5,4);
\draw[ultra thin] (7.5,1) -- (9.5,1);

\draw[ultra thin, blue!40!white] (8, 1) -- (8, 4) node[pos=0.5, black] {$\cell_p$};
\draw[ultra thin, blue!40!white, directed] (8, 4) -- (8, 5.2) node[pos=0.5, black] {$\cell$};
\draw[ultra thin, blue!40!white, directed] (9, 4) -- (9, 5.2) node[pos=0.5, black] {$\cell_s$};
\draw[ultra thin, blue!40!white, directed] (8, 0.5) -- (8, 1);
\draw[ultra thin, blue!40!white, directed] (9, 0.5) -- (9, 1);

\draw[ultra thick, dash pattern=on 10pt off 8pt, red, edge] (9, 1) -- (9, 4) node[pos=0.5, right, black] {$\!\cell_n$};
\draw[ultra thick, red, edge] (9, 1) -- (8, 4);
\end{tikzpicture}
\vspace{-5pt}
\caption{The possible children of a cell $\cell$ in Dijkstra's algorithm. These comprise of all cell-action pairs through which one can arrive at $\cell$ with nonzero probability. (a) a stay-in-lane action from a predecessor $\cell_p$, (b) a forced lane change action into a predecessor $\cell_p$ from a neighbor $\cell_n$ of that predecessor, (c) a lane change action from a predecessor $\cell_p$ into a neighbor $\cell_n$ of that predecessor, and (d) a lane change action into a predecessor $\cell_p$ from a neighbor $\cell_n$ of that predecessor.}
\label{fig:children}
\vspace{-10pt}
\end{figure*}
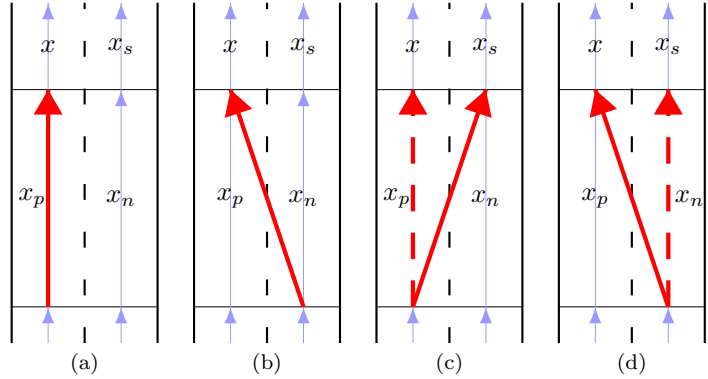

In Algorithm \ref{alg:findchildren} we present a less abstract implementation of lines \ref{line:children} and \ref{line:qvalues} of Algorithm \ref{alg:router} to find the children of a given cell $\cell$ (i.e., the cell-action pairs through which one can arrive at $\cell$ with nonzero probability) and their associated $q$-values. \figref{children} depicts the four types of children a cell $\cell$ may have. In Algorithm \ref{alg:findchildren} we precisely identify the actions by annotating (superscripting) them with the cells one can arrive at through the action. It should be noted that the algorithm will return all cell-action pairs for which one can arrive at $\cell$, but an efficient implementation would only return for each child cell the action with the lowest $q$-value. Also, in line \ref{line:succ} one only needs to consider successors $\cell_s$ that are closed ($\cell_s$ has a finite value and it is not in the queue).

\begin{algorithm}[t]
\caption{\textsc{FindChildren}$(\cell)$}\label{alg:findchildren}
\begin{algorithmic}[1]
\State $\mathcal{C} \gets \varnothing$
\ForAll{$\cell_p \in \textrm{pred}(\cell)$}
\State $\mathcal{C} \gets \mathcal{C} \cup \{(\cell_p, a_s^\cell)\}$ \Comment{\figref{children}(a)}
\State $q(\cell_p, a_s^\cell) \gets c(\cell_p) + g(\cell)$
\ForAll{$\cell_n \in \mathrm{lnb}(\cell_p) \cup \mathrm{rnb}(\cell_p)$}
\State $\mathcal{C} \gets \mathcal{C} \cup \{(\cell_n, a_\flc^\cell)\}$ \Comment{\figref{children}(b)}
\State $q(\cell_n, a_\flc^\cell) \gets c_\lc + c(\cell_n) + (1 - f(\ell(\cell_n)))c_\flc + g(\cell)$
\ForAll{$\cell_s \in \succ(\cell_n)$} \label{line:succ}
\State $\mathcal{C} \gets \mathcal{C} \cup \{(\cell_p, a_\lc^{\cell_s, \cell})\}$ \Comment{\figref{children}(c)}
\State $q(\cell_p, a_\lc^{\cell_s, \cell}) \gets c(\cell_p)+ f(\ell(\cell_p))(c_\lc + g(\cell_s))+ (1 - f(\ell(\cell_p)))g(\cell)$
\State $\mathcal{C} \gets \mathcal{C} \cup \{(\cell_n, a_\lc^{\cell, \cell_s})\}$ \Comment{\figref{children}(d)}
\State $q(\cell_n, a_\lc^{\cell, \cell_s}) \gets c(\cell_n)+ f(\ell(x_n))(c_\lc + g(\cell))+ (1 - f(\ell(x_n)))g(\cell_s)$
\EndFor
\EndFor
\EndFor
\Return{$(\mathcal{C}, q)$}
\end{algorithmic}
\end{algorithm}

The above discussion and \thmref{mdp-mono} imply the following result.
\begin{theorem}
\thmlab{mdp-dijkstra}
Given a lane graph $\mathcal{G} = (\cellset, \edgeset)$ with $\card{\cellset} = n$ and $\card{\edgeset} = O(n)$, a goal cell $\cell_g \in \cellset$, and the property $\forall\cell \in \cellset : c(\cell)/\ell(\cell) \geq \alpha c_\flc$, one can compute the optimal cost-to-go $g(\cell)$ and action $\pi(x)$ for all cells $\cell \in \cellset$ in time $O(n \log n)$.
\end{theorem}
The running time follows from the standard analysis of Dijkstra's algorithm \citep{clrs-alg-09}.

\paragraph{Remark: Visiting multiple goal cells.} For routing of autonomous vehicles, it may be desirable to specify a fixed set of goals that should be visited in order (for example, multiple deliveries within the same neighborhood). Thus, it is of interest to extend the result of \thmref{mdp-dijkstra} to handle multiple goals. Formally, suppose we are given $k$ goals $\cell_1, \ldots, \cell_k \in \cellset$, and we wish to generate a policy which visits these goals in the given order. Practically speaking, $k$ will typically to be much smaller than $n$ (or even constant). To use the result of \thmref{mdp-dijkstra}, one approach is to build an auxiliary graph $\hat{\mathcal{G}} = (\hat{\cellset}, \hat{\edgeset})$, where $\hat{\cellset} = \{(\cell, i) \mid \cell \in \cellset, 0 \leq i \leq k \}$ and for each directed edge $(u, v) \in \edgeset$ and $i = 0, \ldots, k - 1$, we add $((u, i), (v, i))$ to $\hat{\edgeset}$ if $v \neq x_{i+1}$, else add $((u, i), (v, i + 1))$ to $\hat{\edgeset}$. Intuitively, each ``cell'' $(x, i) \in \hat{\cellset}$ corresponds to a cell $x \in \cellset$ in the original lane graph $\mathcal{G}$ along with an integer recording that we have visited the first $i$ goals. We can then run the algorithm of \thmref{mdp-dijkstra} on $\hat{G}$ with the goal node $(x_k, k) \in \hat{\cellset}$. This produces a policy in $O(k n \log n)$ time which visits all goals in the given order.

\section{Heuristics}
\seclab{heuristics}
Since we can successfully employ a variant of Dijkstra's algorithm to solve our routing problem, it is natural to wonder whether we can augment the algorithm with a \emph{heuristic} to focus the search and speed up the algorithm in a similar way A* \citep{hart68} improves upon the running time of Dijkstra's algorithm in deterministic graph search.

The difference between A* and Dijkstra's is that instead of expanding the node $x$ with minimal $g(x)$-value as in Dijkstra's (see line \ref{line:top} of Algorithm \ref{alg:router}), A* expands the node with a minimal value of $g(x) + h(x)$, where $h(x)$ is a given \emph{heuristic} function, that in a goal-initiated search (as in our case) would give an estimate of the cost to reach $x$ from the start node. The heuristic is said to be \emph{consistent} if the following holds:
\begin{align}
    \forall \{x, x' \,|\, x' \in \succ(x)\} : h(x') \leq h(x) + c(x, x'),
\end{align}
where $c(x, x')$ is the edge cost between $x$ and $x'$. If the heuristic is consistent, all nodes will be expanded at most once in the A* algorithm and an optimal path is found in optimal running time \citep{dechter85}. An example of a consistent heuristic is the Euclidean distance when the search problem is situated on a graph embedded in Euclidean space with edge costs equal to their Euclidean length.

Since in our case we have nondeterministic actions, one cannot speak of an ``edge cost'' between two nodes (cells in our case). An appropriate equivalent definition of consistency for our case would be:\footnote{We can use a non-strict inequality here, as long as nodes in the queue with equal $g(x) + h(x)$ are tie-broken by their $g(x)$ value.}
\begin{align}
    \forall \{\cell, \cell' \,|\, p(\cell'|\cell,\pi(\cell)) > 0 \} : h(\cell') \leq h(\cell) + (g(\cell) - g(\cell')).
    \label{eq:consistency}
\end{align}
Note that it follows from the monotonicity requirement of Eq.\ \eqref{eq:monotonicity}, proved in \thmref{mdp-mono}, that $g(\cell) - g(\cell') > 0$.

Since our problem is embedded in the plane (or on the surface of the earth), and the length of the lane cells is the basis for cost, the Euclidean distance (or great circle distance) $d(\cell)$ from the start cell to cell $\cell$ would be a sensible candidate for a consistent heuristic $h(\cell) = d(\cell)$. However, it is not. The only trivial heuristic that is globally consistent appears to be $h(\cell) = 0$ for all $\cell \in \mathcal{\cellset}$, which would render an A*-like algorithm equivalent to the Dijkstra-like algorithm.

We can see this as follows. Let us consider a simple straight section of a highway of length $\ell$ with two lanes, and let the goal point be the end of the left lane, and the start point be the beginning of the right lane. Let the cost of traversing each cell be equal to its length. Let $x$ be the distance away from the goal point along either lane, and let us consider the situation in which we have infinitesimally short cells. Let $g(x)$ be the cost to the goal from the point $x$ along the right lane. It can be shown that $g(x) = x + c_\mathrm{lc} + c_\mathrm{flc} (1 - f(x))$. By choosing the maximum value $c_\mathrm{flc} = 1/\alpha$ for the forced lane change cost, this simplifies to $g(x) = x + c_\mathrm{lc} + \exp(-\alpha x)/\alpha$. It can be seen that the derivative of $g(x)$ at $x = 0$ equals 0. So, we have $g(\Delta x) \approx g(0)$ for small $\Delta x$. If we take as our heuristic the Euclidean distance from the start point, then $h(0) = \ell$ and $h(\Delta x) = \ell - \Delta x$. However, unfortunately, $h(0) \not\leq h(\Delta x) + (g(\Delta x) - g(0))$ as the consistency condition would require. Obviously, $h(x) = g(\cell_\mathrm{start}) - g(x)$ would be globally consistent, but it is not clear that any trivial explicit heuristic exists that is consistent other than $h(x) = 0$ when the length of cells can be arbitrarily small.

\subsection{A Consistent Heuristic}

Our situation is in a sense similar to that of the Fast Marching Method (FMM) \citep{sethian1996fast}, which implements Dijkstra's algorithm to propagate the Eikonal equation on a grid. In the standard case of a 4-connected grid, FMM does not allow for any consistent heuristic other than $h(x) = 0$ \citep{ccv-cdree-14,yl-sdaaofp-11}. However for 6-connected and 8-connected grids, one can get a modest consistent heuristic of the form $h(x) = \lambda d(x)$, where $\lambda \in [0, 1]$ and $d(x)$ is the Euclidean distance from state $x$ to a desired starting state \citep{ccv-cdree-14, yl-sdaaofp-11, yl-sdafg-12}. Specifically, the heuristic is the Euclidean distance to the start state, scaled by $\lambda = 1/\sqrt{2}$ for the 8-connected grid, and $\lambda = 1/2$ for the 6-connected grid. A number of papers appear in the literature that suggest to use heuristics in FMM, without discussing whether any notion of consistency is satisfied (e.g.\ \cite{ferguson05}). Using inconsistent heuristics would lead either to the computation of incorrect values, or the revisiting of nodes an unbounded number of times. The latter could negate any potential speed-up that an A*-like approach may provide, although works like \cite{felner2011inconsistent} show that in certain cases inconsistent heuristics may still be beneficial.


In the previous example we saw that for infinitesimally short cells the only consistent heuristic has $\lambda = 0$. This does, however, suggest that there may exist a $\lambda > 0$ that results in a consistent heuristic for a non-zero minimum cell length. In this section we focus on obtaining a consistent heuristic when the cell length in the lane graph is bounded from below. Naturally, we want $\lambda$ as close to one as possible to guide the search. In our case, the maximum $\lambda$ while retaining consistency appears to be a function of the product of the minimum cell length and the parameter $\alpha$.

To begin, we make the following assumptions for all $\cell \in \cellset$:
\begin{align}
c(\cell) & \geq \ell(\cell) \\
c_{\flc} & = 1 / \alpha \\
c_{\lc} & \geq w(\cell),
\end{align}
where $w(\cell)$ is the width of a cell (the distance between the endpoint of the cell and its neighbor's endpoint). In practice, $w(\cell)$ is typically bounded by a constant as neighboring driving lanes are sufficiently close together.\footnote{For example, the lane graph used at Nuro considers two cells to be related by a left or right neighboring relationship only if they are within 10 meters of each other.} The first two assumptions formalize that the cost of a cell is based on its length, and that the monotonicity requirement is satisfied.

For a cell $\cell$, define the function
\begin{align*}
\lambda(\cell) &= 1 - \frac{f(\ell(\cell))}{\alpha \ell(\cell)} = 1 - \frac{1 - \exp(-\alpha \ell(\cell))}{\alpha \ell(\cell)}.
\end{align*}
Note that $0 < \lambda(\cell) < 1$. The heuristic is:
\begin{align}
h(\cell) &= \lambda d(\cell)  \quad \cell \in \cellset, \label{eq:heuristic}
\end{align}
where $d(x)$ is the Euclidean distance from the start of cell $\cell$ to the start of the cell $x_{\mathrm{start}}$ and
\begin{align}
\lambda = \min_{\cell \in \cellset} \lambda(\cell).
\end{align}

We will derive the expression for $\lambda$ and prove consistency of this heuristic below (note that the heuristic is also admissible). Starting at the goal cell, we can now compute the heuristics for all cells as we visit them and add them to the queue in our Dijkstra-like algorithm, where nodes are expanded with priority $g(\cell) + h(\cell)$. Note that if the length of the shortest cell is small, then $\lambda$ will be small, and the heuristic will do little to focus the search. Practically speaking, it is possible that lane graph cell lengths are highly dependent on the road geometry, and thus the minimum cell length is generally very short (in our real-world road networks, cells can be only a few centimeters long). However, in lane graphs with longer cells (for example, long stretches of a highway) $\lambda$ will be greater and the algorithm benefits more from the heuristic.

\begin{lemma}
\lemlab{min-len}
The heuristic of Eq. \eqref{eq:heuristic} is consistent, i.e. it satisfies the requirement of Eq.\ \eqref{eq:consistency}.
\end{lemma}
\begin{proof}
Similar to the proof of \lemref{lane-change-mono}, we break the proof down into various cases, depending on the optimal action $\pi(\cell)$ to take at each cell $\cell$.
\begin{enumerate}[(a)]
    \item Stay-in-lane actions. The optimal action at a cell $\cell$ is a stay-in-lane action leading to a cell $\cell'$, which is a direct successor of $\cell$ in the lane graph.
    From the triangle inequality, we have that:
    $$d(\cell') - d(\cell) \leq \ell(x).$$

    Thus, it follows that:
\begin{align*}
h(x') - h(x) = \lambda(d(x') - d(x)) \leq \lambda\ell(\cell).
\end{align*}
Comparing with Eq.\ \eqref{eq:consistency}, we see that it suffices to show that $$g(x) - g(x') \geq \lambda \ell(\cell)$$ to show that the heuristic is consistent for stay-in-lane actions.

    From Eq.\ \eqref{eq:sil}, we have that $g(\cell) - g(\cell') = c(\cell)$ for a stay-in-lane action. Since $c(\cell) \geq \ell(\cell)$ and $\lambda \leq 1$, the above requirement is indeed satisfied.

    \item Forced lane change actions. The optimal action at a cell $\cell$ is a forced lane change action leading to a cell $\cell'$, which is a neighbor of a direct successor of $\cell$ in the lane graph.
    From the triangle inequality, we have that:
    $$d(\cell') - d(\cell) \leq \ell(x) + w(x).$$

Comparing with Eq.\ \eqref{eq:consistency} (and similar to the previous case), we see that it suffices to show that $$g(x) - g(x') \geq \lambda(\cell) (\ell(\cell) + w(\cell))$$ to show that the heuristic is consistent for stay-in-lane actions.

    From Eq.\ \eqref{eq:flc}, we have that $g(\cell) - g(\cell') \geq c(\cell) + c_{\lc}$ for a forced lane change action. Since $c(\cell) \geq \ell(\cell)$, $c_{\lc} \geq w(\cell)$, and $\lambda \leq 1$, the above requirement is indeed satisfied.

    \item Lane change actions. The optimal action at a cell $\cell$ is a forced lane change action leading to to either to a cell $\cell_{ns}$ (if the lane change is successful), which is a neighbor of direct successor of $\cell$ in the lane graph, or to a cell $\cell_s$ (if unsuccessful), which is a direct successor of $\cell$ in the lane graph.
    From the triangle inequality, we have that:
    \begin{align}
    d(\cell_{ns}) - d(\cell) & \leq \ell(x) + w(x). \nonumber \\
    d(\cell_s) - d(\cell) & \leq \ell(x). \nonumber
    \end{align}
    And since the lane change action can lead to two cells ($\cell_{ns}$ and $\cell_s$), we have
    \begin{align*}
      h(x_{ns}) - h(x) &\leq \lambda (\ell(x) + w(x)) \quad \text{and} \\
      h(x_s) - h(x) &\leq \lambda \ell(x).
    \end{align*}
    Comparing with Eq.\ \eqref{eq:consistency}, we see that it suffices to show that
    \begin{align}
    g(x) - g(x_{ns}) & \geq \lambda (\ell(\cell) + w(\cell))\quad\text{and} ~ \label{eq:req1} \\
    g(x) - g(x_s) & \geq \lambda \ell(\cell) \label{eq:req2}
    \end{align}
    to show that the heuristic is consistent for lane change actions.

    We can show that these requirements are satisfied using the bounds from Eqs.\ \eqref{eq:inequality_gns} and \eqref{eq:inequality_gs}:
    \begin{align}
        c_{\lc} \leq g(\cell_s) - g(\cell_{ns}) \leq c_{\lc} + c_{\flc}.
        \label{eq:g-bounds}
    \end{align}
    The requirement of Eq.\ \eqref{eq:req1} is satisfied as:
    \begin{align*}
        g(\cell) - g(\cell_{ns})
        &= f(\ell(\cell))(c_{\lc} + c(\cell) + g(\cell_{ns})) + (1 - f(\ell(\cell)))(c(\cell) + g(\cell_s)) - g(\cell_{ns}) \\
        &= c(\cell) + f(\ell(\cell))c_{\lc} + (1 - f(\ell(\cell)))(g(\cell_{s}) - g(\cell_{ns})) \\
        &\geq c(\cell) + c_{\lc}, \\
        &\geq \lambda (\ell(x) + w(x))
    \end{align*}
    where the before-last inequality follows from the lower bound in Eq.\ \eqref{eq:g-bounds}, and the last inequality follows since $c(\cell) \geq \ell(\cell)$, $c_{\lc} \geq w(\cell)$, and $\lambda \leq 1$.

    The requirement of Eq.\ \eqref{eq:req2} is satisfied as:
    \begin{align*}
        g(\cell) - g(\cell_{s})
        &= f(\ell(\cell))(c_{\lc} + c(\cell) + g(\cell_{ns})) + (1 - f(\ell(\cell)))(c(\cell) + g(\cell_s)) - g(\cell_{s}) \\
        &= c(\cell) + f(\ell(\cell)) c_{\lc} - f(\ell(\cell))(g(\cell_s) - g(\cell_{ns})) \\
        &\geq c(\cell) + f(\ell(\cell)) c_{\lc} - f(\ell(\cell))(c_{\flc} + c_{\lc}) \\
        &= c(\cell) - f(\ell(\cell))c_{\flc} \\
        & \geq \ell(\cell) - f(\ell(\cell))c_{\flc} \\
        & = \left(1 - \frac{f(\ell(\cell))}{\alpha \ell(\cell)}\right)\ell(\cell) = \lambda(\cell)\ell(\cell) \geq \lambda \ell(\cell),
    \end{align*}
    where the first inequality follows from the upper bound in Eq.\ \eqref{eq:g-bounds} and the second inequality follows as $c(\cell) \geq \ell(\cell)$. The before-last equality follows from the assumption that $c_{\flc} = 1/\alpha$.
\end{enumerate}
The cases combined show that the heuristic we propose is consistent.
\qed
\end{proof}

\section{A single-path maximum-likelihood route representation}
\seclab{max_probability_path}
While the computed policy describes the optimal cost and action to take from any given cell, some applications may instead call for a single sequence of consecutive cells connecting an initial cell to a goal cell. For instance, a passenger of an autonomous vehicle may not need to be shown the entire policy in order to convey the route the vehicle is likely to take. Instead, it may be desirable to construct a single line to overlay on a map of the road network to visualize a possible route the vehicle may take from its current location to the goal point. Of course in the presence of stochastic lane change actions there is no guarantee the vehicle will be able to actually follow such a single line representation of the route---local traffic conditions may not permit it to do so. However, we can construct the sequence of cells which the vehicle is \textit{most likely} to traverse. To extract such a maximum probability path we must account for the success probability of each lane change action in the sequence of cells traversed.

Concretely, suppose an autonomous vehicle is currently traversing a cell\footnote{For the sake of exposition, we assume the cell that the autonomous vehicle is currently traversing is readily available. Note that this task is not a trivial problem, but is outside of scope for this paper.} $\cell$ and following an optimal policy $\pi$. Define the function $q_{\cell, \pi} : \cellset \to [0, 1]$, where $q_{\cell, \pi}(\cellB)$ is the probability of at any time traversing cell $\cellB$ under $\pi$ when initially starting at $\cell$. Recall that monotonicity implies that we can represent the policy as a directed acyclic graph. We can compute $q_{\cell, \pi}$ by first topologically sorting this policy graph such that its first element is the start cell $\cell$ and its final element is the goal cell $\cell_g$.

In general, we may arrive at a cell in multiple ways: We may enter it by following its predecessor cell, or by lane changing into it from a neighbor of its predecessor. Thus, we must compute the total probability of entering each cell by summing up all the contributions due to the different ways we may reach it. We can do so by setting initial values of $q_{\cell, \pi}$, starting with $q_{\cell, \pi}(\cell) = 1$ and $q_{\cell, \pi}(\cellC) = 0$ for all other cells $\cellC \neq \cell$. Starting at $\cell$ we now traverse the graph in its topological ordering and for each cell $\cellC$ we may enter from our current cell $\cellB$ we add the probability of arriving in $\cellC$ from $\cellB$ under policy $\pi$ to the total probability $q_{\cell, \pi}(\cellC)$.

If at $\cellB$ the optimal action is to stay in lane there is only one cell we may enter next, which is its successor $\cellB'$. Thus, we simply add $q_{\cell, \pi}(\cellB)$ to $q_{\cell, \pi}(\cellB')$. Alternatively, if the optimal action from $\cellB$ is a lane change the vehicle could lane change into cell $\cellB''$ or enter its successor $\cellB'$. Thus, we add $f(\ell(\cellB))q_{\cell, \pi}(\cellB)$ to $q_{\cell, \pi}(\cellB'')$ and $(1 - f(\ell(\cell)) q_{\cell, \pi}(\cellB)$ to $q_{\cellB, \pi}(\cellB')$. This dependence of the traversal probability of a cell on the traversal probabilities of the cells that came before it illustrate the need for the topological order in which we process the cells.

Note that in the above we have not yet discussed forced lane change actions. It may be desirable to handle forced lane changes much like stay-in-lane actions in that there is only a single possible target cell and we have a lane change success probability of one. This may be the case if our motion planner will treat forced lane changes as binding: If traffic does not allow a lane change even though the policy demands it, we must come to a stop and wait until it does. In practice, it is often preferable to instead handle forced lane changes like simple lane change actions, allowing the motion planner to miss a lane change if traffic does not permit it. This is a reasonable choice on real road networks as it is generally possible to reach any given goal from any given start cell and thus missing a lane change will not result in the vehicle not being able to reach the goal anymore.

Once we have computed the traversal probabilities of all nodes in our policy graph we can use them to identify the maximum likelihood sequence of cells connecting the start cell to the goal cell: Starting at the initial cell, continuously and greedily select the cell from those reachable that has the largest traversal probability until we reach the goal cell. The cells visited as such make up the maximum likelihood cell sequence. We will use this method to simplify the visualization of routes on actual road maps in the next section.

\section{Experiments}
\seclab{experiments}
\begin{figure*}[t]
\centering
\begin{tabular}{c}
\begin{tikzpicture}[scale=0.45]
\draw[thick] (0,3) -- (20,3);
\draw[thick] (0,0) -- (14,0);
\draw[thick] (16,0) -- (20,0);
\draw[thick] (14,-1) -- (16,0);
\draw[ultra thin] (14,-1) -- (14,0);
\draw[thick] (12,-1) -- (14,-1);
\draw[ultra thin] (12,-1) -- (12,0);
\draw[thick, dash pattern=on 6pt off 12pt] (0,1) -- (20,1);
\draw[thick, dash pattern=on 6pt off 12pt] (0,2) -- (20,2);
\draw[ultra thin] (0,0) -- (0,3);
\draw[ultra thin] (2,0) -- (2,3);
\draw[ultra thin] (4,0) -- (4,3);
\draw[ultra thin] (6,0) -- (6,3);
\draw[ultra thin] (8,0) -- (8,3);
\draw[ultra thin] (10,0) -- (10,3);
\draw[ultra thin] (12,0) -- (12,3);
\draw[ultra thin] (14,0) -- (14,3);
\draw[ultra thin] (16,0) -- (16,3);
\draw[ultra thin] (18,0) -- (18,3);
\draw[ultra thin] (20,0) -- (20,3);
\draw[thick, red, edge2] (0,0.5) -- (2,0.5);
\draw[thick, red, edge2] (0,1.5) -- (2,0.5);
\draw[thick, dash pattern=on 3pt off 3pt , red, edge2] (0,1.5) -- (2,1.5);
\draw[thick, red, edge2] (0,2.5) -- (2,1.5);
\draw[thick, dash pattern=on 3pt off 3pt , red, edge2] (0,2.5) -- (2,2.5);
\draw[thick, red, edge2] (2,0.5) -- (4,0.5);
\draw[thick, red, edge2] (2,1.5) -- (4,0.5);
\draw[thick, dash pattern=on 3pt off 3pt , red, edge2] (2,1.5) -- (4,1.5);
\draw[thick, red, edge2] (2,2.5) -- (4,1.5);
\draw[thick, dash pattern=on 3pt off 3pt , red, edge2] (2,2.5) -- (4,2.5);
\draw[thick, red, edge2] (4,0.5) -- (6,0.5);
\draw[thick, red, edge2] (4,1.5) -- (6,0.5);
\draw[thick, dash pattern=on 3pt off 3pt , red, edge2] (4,1.5) -- (6,1.5);
\draw[thick, red, edge2] (4,2.5) -- (6,1.5);
\draw[thick, dash pattern=on 3pt off 3pt , red, edge2] (4,2.5) -- (6,2.5);
\draw[thick, red, edge2] (6,0.5) -- (8,0.5);
\draw[thick, red, edge2] (6,1.5) -- (8,0.5);
\draw[thick, dash pattern=on 3pt off 3pt , red, edge2] (6,1.5) -- (8,1.5);
\draw[thick, red, edge2] (6,2.5) -- (8,1.5);
\draw[thick, dash pattern=on 3pt off 3pt , red, edge2] (6,2.5) -- (8,2.5);
\draw[thick, red, edge2] (8,0.5) -- (10,0.5);
\draw[thick, red, edge2] (8,1.5) -- (10,0.5);
\draw[thick, dash pattern=on 3pt off 3pt , red, edge2] (8,1.5) -- (10,1.5);
\draw[thick, red, edge2] (8,2.5) -- (10,1.5);
\draw[thick, dash pattern=on 3pt off 3pt , red, edge2] (8,2.5) -- (10,2.5);
\draw[thick, red, edge2] (10,0.5) -- (12,0.5);
\draw[thick, red, edge2] (10,1.5) -- (12,0.5);
\draw[thick, dash pattern=on 3pt off 3pt , red, edge2] (10,1.5) -- (12,1.5);
\draw[thick, red, edge2] (10,2.5) -- (12,1.5);
\draw[thick, dash pattern=on 3pt off 3pt , red, edge2] (10,2.5) -- (12,2.5);
\draw[thick, red, edge2] (12,0.5) -- (14,0.5);
\draw[thick, red, edge2] (12,1.5) -- (14,0.5);
\draw[thick, dash pattern=on 3pt off 3pt , red, edge2] (12,1.5) -- (14,1.5);
\draw[thick, red, edge2] (12,2.5) -- (14,1.5);
\draw[thick, dash pattern=on 3pt off 3pt , red, edge2] (12,2.5) -- (14,2.5);
\draw[thick, red, edge2] (14,0.5) -- (16,0.5);
\draw[thick, red, edge2] (14,1.5) -- (16,0.5);
\draw[thick, dash pattern=on 3pt off 3pt , red, edge2] (14,1.5) -- (16,1.5);
\draw[thick, red, edge2] (14,2.5) -- (16,1.5);
\draw[thick, dash pattern=on 3pt off 3pt , red, edge2] (14,2.5) -- (16,2.5);
\draw[thick, red, edge2] (16,0.5) -- (18,0.5);
\draw[thick, red, edge2] (16,1.5) -- (18,0.5);
\draw[thick, dash pattern=on 3pt off 3pt , red, edge2] (16,1.5) -- (18,1.5);
\draw[thick, red, edge2] (16,2.5) -- (18,1.5);
\draw[thick, dash pattern=on 3pt off 3pt , red, edge2] (16,2.5) -- (18,2.5);
\draw[thick, red, edge2] (18,0.5) -- (20,0.5);
\draw[thick, red, edge2] (18,1.5) -- (20,0.5);
\draw[thick, dash pattern=on 3pt off 3pt , red, edge2] (18,1.5) -- (20,1.5);
\draw[thick, red, edge2] (18,2.5) -- (20,1.5);
\draw[thick, dash pattern=on 3pt off 3pt , red, edge2] (18,2.5) -- (20,2.5);
\draw[thick, red, edge2] (12,-0.5) -- (14,-0.5);
\draw[thick, red, edge2] (14,-0.5) -- (16,0.5);
\end{tikzpicture} \\
{\footnotesize (a) $c_\textrm{lc} = 5$, $c_\textrm{merge}=0$, $\alpha=0.01$, $c_\textrm{left} = 0.1$} \\
\begin{tikzpicture}[scale=0.45]
\draw[thick] (0,3) -- (20,3);
\draw[thick] (0,0) -- (14,0);
\draw[thick] (16,0) -- (20,0);
\draw[thick] (14,-1) -- (16,0);
\draw[ultra thin] (14,-1) -- (14,0);
\draw[thick] (12,-1) -- (14,-1);
\draw[ultra thin] (12,-1) -- (12,0);
\draw[thick, dash pattern=on 6pt off 12pt] (0,1) -- (20,1);
\draw[thick, dash pattern=on 6pt off 12pt] (0,2) -- (20,2);
\draw[ultra thin] (0,0) -- (0,3);
\draw[ultra thin] (2,0) -- (2,3);
\draw[ultra thin] (4,0) -- (4,3);
\draw[ultra thin] (6,0) -- (6,3);
\draw[ultra thin] (8,0) -- (8,3);
\draw[ultra thin] (10,0) -- (10,3);
\draw[ultra thin] (12,0) -- (12,3);
\draw[ultra thin] (14,0) -- (14,3);
\draw[ultra thin] (16,0) -- (16,3);
\draw[ultra thin] (18,0) -- (18,3);
\draw[ultra thin] (20,0) -- (20,3);
\draw[thick, red, edge2] (0,0.5) -- (2,1.5);
\draw[thick, dash pattern=on 3pt off 3pt , red, edge2] (0,0.5) -- (2,0.5);
\draw[thick, red, edge2] (0,1.5) -- (2,1.5);
\draw[thick, red, edge2] (0,2.5) -- (2,1.5);
\draw[thick, dash pattern=on 3pt off 3pt , red, edge2] (0,2.5) -- (2,2.5);
\draw[thick, red, edge2] (2,0.5) -- (4,1.5);
\draw[thick, dash pattern=on 3pt off 3pt , red, edge2] (2,0.5) -- (4,0.5);
\draw[thick, red, edge2] (2,1.5) -- (4,1.5);
\draw[thick, red, edge2] (2,2.5) -- (4,1.5);
\draw[thick, dash pattern=on 3pt off 3pt , red, edge2] (2,2.5) -- (4,2.5);
\draw[thick, red, edge2] (4,0.5) -- (6,1.5);
\draw[thick, dash pattern=on 3pt off 3pt , red, edge2] (4,0.5) -- (6,0.5);
\draw[thick, red, edge2] (4,1.5) -- (6,1.5);
\draw[thick, red, edge2] (4,2.5) -- (6,1.5);
\draw[thick, dash pattern=on 3pt off 3pt , red, edge2] (4,2.5) -- (6,2.5);
\draw[thick, red, edge2] (6,0.5) -- (8,1.5);
\draw[thick, dash pattern=on 3pt off 3pt , red, edge2] (6,0.5) -- (8,0.5);
\draw[thick, red, edge2] (6,1.5) -- (8,1.5);
\draw[thick, red, edge2] (6,2.5) -- (8,1.5);
\draw[thick, dash pattern=on 3pt off 3pt , red, edge2] (6,2.5) -- (8,2.5);
\draw[thick, red, edge2] (8,0.5) -- (10,1.5);
\draw[thick, dash pattern=on 3pt off 3pt , red, edge2] (8,0.5) -- (10,0.5);
\draw[thick, red, edge2] (8,1.5) -- (10,1.5);
\draw[thick, red, edge2] (8,2.5) -- (10,1.5);
\draw[thick, dash pattern=on 3pt off 3pt , red, edge2] (8,2.5) -- (10,2.5);
\draw[thick, red, edge2] (10,0.5) -- (12,1.5);
\draw[thick, dash pattern=on 3pt off 3pt , red, edge2] (10,0.5) -- (12,0.5);
\draw[thick, red, edge2] (10,1.5) -- (12,1.5);
\draw[thick, red, edge2] (10,2.5) -- (12,1.5);
\draw[thick, dash pattern=on 3pt off 3pt , red, edge2] (10,2.5) -- (12,2.5);
\draw[thick, red, edge2] (12,0.5) -- (14,1.5);
\draw[thick, dash pattern=on 3pt off 3pt , red, edge2] (12,0.5) -- (14,0.5);
\draw[thick, red, edge2] (12,1.5) -- (14,1.5);
\draw[thick, red, edge2] (12,2.5) -- (14,1.5);
\draw[thick, dash pattern=on 3pt off 3pt , red, edge2] (12,2.5) -- (14,2.5);
\draw[thick, red, edge2] (14,0.5) -- (16,0.5);
\draw[thick, red, edge2] (14,1.5) -- (16,1.5);
\draw[thick, red, edge2] (14,2.5) -- (16,1.5);
\draw[thick, dash pattern=on 3pt off 3pt , red, edge2] (14,2.5) -- (16,2.5);
\draw[thick, red, edge2] (16,0.5) -- (18,0.5);
\draw[thick, red, edge2] (16,1.5) -- (18,0.5);
\draw[thick, dash pattern=on 3pt off 3pt , red, edge2] (16,1.5) -- (18,1.5);
\draw[thick, red, edge2] (16,2.5) -- (18,1.5);
\draw[thick, dash pattern=on 3pt off 3pt , red, edge2] (16,2.5) -- (18,2.5);
\draw[thick, red, edge2] (18,0.5) -- (20,0.5);
\draw[thick, red, edge2] (18,1.5) -- (20,0.5);
\draw[thick, dash pattern=on 3pt off 3pt , red, edge2] (18,1.5) -- (20,1.5);
\draw[thick, red, edge2] (18,2.5) -- (20,1.5);
\draw[thick, dash pattern=on 3pt off 3pt , red, edge2] (18,2.5) -- (20,2.5);
\draw[thick, red, edge2] (12,-0.5) -- (14,-0.5);
\draw[thick, red, edge2] (14,-0.5) -- (16,0.5);
\end{tikzpicture} \\
{\footnotesize (b) $c_\textrm{lc} = 5$, $c_\textrm{merge}=50$, $\alpha=0.01$, $c_\textrm{left} = 0.1$} \\
\begin{tikzpicture}[scale=0.45]
\draw[thick] (0,3) -- (20,3);
\draw[thick] (0,0) -- (14,0);
\draw[thick] (16,0) -- (20,0);
\draw[thick] (14,-1) -- (16,0);
\draw[ultra thin] (14,-1) -- (14,0);
\draw[thick] (12,-1) -- (14,-1);
\draw[ultra thin] (12,-1) -- (12,0);
\draw[thick, dash pattern=on 6pt off 12pt] (0,1) -- (20,1);
\draw[thick, dash pattern=on 6pt off 12pt] (0,2) -- (20,2);
\draw[ultra thin] (0,0) -- (0,3);
\draw[ultra thin] (2,0) -- (2,3);
\draw[ultra thin] (4,0) -- (4,3);
\draw[ultra thin] (6,0) -- (6,3);
\draw[ultra thin] (8,0) -- (8,3);
\draw[ultra thin] (10,0) -- (10,3);
\draw[ultra thin] (12,0) -- (12,3);
\draw[ultra thin] (14,0) -- (14,3);
\draw[ultra thin] (16,0) -- (16,3);
\draw[ultra thin] (18,0) -- (18,3);
\draw[ultra thin] (20,0) -- (20,3);
\draw[thick, red, edge2] (0,0.5) -- (2,0.5);
\draw[thick, red, edge2] (0,1.5) -- (2,1.5);
\draw[thick, red, edge2] (0,2.5) -- (2,1.5);
\draw[thick, dash pattern=on 3pt off 3pt , red, edge2] (0,2.5) -- (2,2.5);
\draw[thick, red, edge2] (2,0.5) -- (4,0.5);
\draw[thick, red, edge2] (2,1.5) -- (4,1.5);
\draw[thick, red, edge2] (2,2.5) -- (4,1.5);
\draw[thick, dash pattern=on 3pt off 3pt , red, edge2] (2,2.5) -- (4,2.5);
\draw[thick, red, edge2] (4,0.5) -- (6,0.5);
\draw[thick, red, edge2] (4,1.5) -- (6,1.5);
\draw[thick, red, edge2] (4,2.5) -- (6,1.5);
\draw[thick, dash pattern=on 3pt off 3pt , red, edge2] (4,2.5) -- (6,2.5);
\draw[thick, red, edge2] (6,0.5) -- (8,1.5);
\draw[thick, dash pattern=on 3pt off 3pt , red, edge2] (6,0.5) -- (8,0.5);
\draw[thick, red, edge2] (6,1.5) -- (8,1.5);
\draw[thick, red, edge2] (6,2.5) -- (8,1.5);
\draw[thick, dash pattern=on 3pt off 3pt , red, edge2] (6,2.5) -- (8,2.5);
\draw[thick, red, edge2] (8,0.5) -- (10,1.5);
\draw[thick, dash pattern=on 3pt off 3pt , red, edge2] (8,0.5) -- (10,0.5);
\draw[thick, red, edge2] (8,1.5) -- (10,1.5);
\draw[thick, red, edge2] (8,2.5) -- (10,1.5);
\draw[thick, dash pattern=on 3pt off 3pt , red, edge2] (8,2.5) -- (10,2.5);
\draw[thick, red, edge2] (10,0.5) -- (12,1.5);
\draw[thick, dash pattern=on 3pt off 3pt , red, edge2] (10,0.5) -- (12,0.5);
\draw[thick, red, edge2] (10,1.5) -- (12,1.5);
\draw[thick, red, edge2] (10,2.5) -- (12,1.5);
\draw[thick, dash pattern=on 3pt off 3pt , red, edge2] (10,2.5) -- (12,2.5);
\draw[thick, red, edge2] (12,0.5) -- (14,1.5);
\draw[thick, dash pattern=on 3pt off 3pt , red, edge2] (12,0.5) -- (14,0.5);
\draw[thick, red, edge2] (12,1.5) -- (14,1.5);
\draw[thick, red, edge2] (12,2.5) -- (14,1.5);
\draw[thick, dash pattern=on 3pt off 3pt , red, edge2] (12,2.5) -- (14,2.5);
\draw[thick, red, edge2] (14,0.5) -- (16,0.5);
\draw[thick, red, edge2] (14,1.5) -- (16,1.5);
\draw[thick, red, edge2] (14,2.5) -- (16,1.5);
\draw[thick, dash pattern=on 3pt off 3pt , red, edge2] (14,2.5) -- (16,2.5);
\draw[thick, red, edge2] (16,0.5) -- (18,0.5);
\draw[thick, red, edge2] (16,1.5) -- (18,0.5);
\draw[thick, dash pattern=on 3pt off 3pt , red, edge2] (16,1.5) -- (18,1.5);
\draw[thick, red, edge2] (16,2.5) -- (18,1.5);
\draw[thick, dash pattern=on 3pt off 3pt , red, edge2] (16,2.5) -- (18,2.5);
\draw[thick, red, edge2] (18,0.5) -- (20,0.5);
\draw[thick, red, edge2] (18,1.5) -- (20,0.5);
\draw[thick, dash pattern=on 3pt off 3pt , red, edge2] (18,1.5) -- (20,1.5);
\draw[thick, red, edge2] (18,2.5) -- (20,1.5);
\draw[thick, dash pattern=on 3pt off 3pt , red, edge2] (18,2.5) -- (20,2.5);
\draw[thick, red, edge2] (12,-0.5) -- (14,-0.5);
\draw[thick, red, edge2] (14,-0.5) -- (16,0.5);
\end{tikzpicture} \\
{\footnotesize(c) $c_\textrm{lc} = 5$, $c_\textrm{merge}=50$, $\alpha=0.01$, $c_\textrm{left} = 0.25$} \\
\begin{tikzpicture}[scale=0.45]
\draw[thick] (0,3) -- (20,3);
\draw[thick] (0,0) -- (14,0);
\draw[thick] (16,0) -- (20,0);
\draw[thick] (14,-1) -- (16,0);
\draw[ultra thin] (14,-1) -- (14,0);
\draw[thick] (12,-1) -- (14,-1);
\draw[ultra thin] (12,-1) -- (12,0);
\draw[thick, dash pattern=on 6pt off 12pt] (0,1) -- (20,1);
\draw[thick, dash pattern=on 6pt off 12pt] (0,2) -- (20,2);
\draw[ultra thin] (0,0) -- (0,3);
\draw[ultra thin] (2,0) -- (2,3);
\draw[ultra thin] (4,0) -- (4,3);
\draw[ultra thin] (6,0) -- (6,3);
\draw[ultra thin] (8,0) -- (8,3);
\draw[ultra thin] (10,0) -- (10,3);
\draw[ultra thin] (12,0) -- (12,3);
\draw[ultra thin] (14,0) -- (14,3);
\draw[ultra thin] (16,0) -- (16,3);
\draw[ultra thin] (18,0) -- (18,3);
\draw[ultra thin] (20,0) -- (20,3);
\draw[thick, red, edge2] (0,0.5) -- (2,0.5);
\draw[thick, red, edge2] (0,1.5) -- (2,1.5);
\draw[thick, red, edge2] (0,2.5) -- (2,1.5);
\draw[thick, dash pattern=on 3pt off 3pt , red, edge2] (0,2.5) -- (2,2.5);
\draw[thick, red, edge2] (2,0.5) -- (4,0.5);
\draw[thick, red, edge2] (2,1.5) -- (4,1.5);
\draw[thick, red, edge2] (2,2.5) -- (4,1.5);
\draw[thick, dash pattern=on 3pt off 3pt , red, edge2] (2,2.5) -- (4,2.5);
\draw[thick, red, edge2] (4,0.5) -- (6,0.5);
\draw[thick, red, edge2] (4,1.5) -- (6,1.5);
\draw[thick, red, edge2] (4,2.5) -- (6,1.5);
\draw[thick, dash pattern=on 3pt off 3pt , red, edge2] (4,2.5) -- (6,2.5);
\draw[thick, red, edge2] (6,0.5) -- (8,0.5);
\draw[thick, red, edge2] (6,1.5) -- (8,1.5);
\draw[thick, red, edge2] (6,2.5) -- (8,1.5);
\draw[thick, dash pattern=on 3pt off 3pt , red, edge2] (6,2.5) -- (8,2.5);
\draw[thick, red, edge2] (8,0.5) -- (10,0.5);
\draw[thick, red, edge2] (8,1.5) -- (10,1.5);
\draw[thick, red, edge2] (8,2.5) -- (10,1.5);
\draw[thick, dash pattern=on 3pt off 3pt , red, edge2] (8,2.5) -- (10,2.5);
\draw[thick, red, edge2] (10,0.5) -- (12,0.5);
\draw[thick, red, edge2] (10,1.5) -- (12,1.5);
\draw[thick, red, edge2] (10,2.5) -- (12,1.5);
\draw[thick, dash pattern=on 3pt off 3pt , red, edge2] (10,2.5) -- (12,2.5);
\draw[thick, red, edge2] (12,0.5) -- (14,0.5);
\draw[thick, red, edge2] (12,1.5) -- (14,1.5);
\draw[thick, red, edge2] (12,2.5) -- (14,1.5);
\draw[thick, dash pattern=on 3pt off 3pt , red, edge2] (12,2.5) -- (14,2.5);
\draw[thick, red, edge2] (14,0.5) -- (16,0.5);
\draw[thick, red, edge2] (14,1.5) -- (16,1.5);
\draw[thick, red, edge2] (14,2.5) -- (16,1.5);
\draw[thick, dash pattern=on 3pt off 3pt , red, edge2] (14,2.5) -- (16,2.5);
\draw[thick, red, edge2] (16,0.5) -- (18,0.5);
\draw[thick, red, edge2] (16,1.5) -- (18,0.5);
\draw[thick, dash pattern=on 3pt off 3pt , red, edge2] (16,1.5) -- (18,1.5);
\draw[thick, red, edge2] (16,2.5) -- (18,1.5);
\draw[thick, dash pattern=on 3pt off 3pt , red, edge2] (16,2.5) -- (18,2.5);
\draw[thick, red, edge2] (18,0.5) -- (20,0.5);
\draw[thick, red, edge2] (18,1.5) -- (20,0.5);
\draw[thick, dash pattern=on 3pt off 3pt , red, edge2] (18,1.5) -- (20,1.5);
\draw[thick, red, edge2] (18,2.5) -- (20,1.5);
\draw[thick, dash pattern=on 3pt off 3pt , red, edge2] (18,2.5) -- (20,2.5);
\draw[thick, red, edge2] (12,-0.5) -- (14,-0.5);
\draw[thick, red, edge2] (14,-0.5) -- (16,0.5);
\end{tikzpicture} \\
{\footnotesize(d) $c_\textrm{lc} = 10$, $c_\textrm{merge}=25$, $\alpha=0.01$, $c_\textrm{left} = 0.1$}\\
\begin{tikzpicture}[scale=0.45]
\draw[thick] (0,3) -- (20,3);
\draw[thick] (0,0) -- (14,0);
\draw[thick] (16,0) -- (20,0);
\draw[thick] (14,-1) -- (16,0);
\draw[ultra thin] (14,-1) -- (14,0);
\draw[thick] (12,-1) -- (14,-1);
\draw[ultra thin] (12,-1) -- (12,0);
\draw[thick, dash pattern=on 6pt off 12pt] (0,1) -- (20,1);
\draw[thick, dash pattern=on 6pt off 12pt] (0,2) -- (20,2);
\draw[ultra thin] (0,0) -- (0,3);
\draw[ultra thin] (2,0) -- (2,3);
\draw[ultra thin] (4,0) -- (4,3);
\draw[ultra thin] (6,0) -- (6,3);
\draw[ultra thin] (8,0) -- (8,3);
\draw[ultra thin] (10,0) -- (10,3);
\draw[ultra thin] (12,0) -- (12,3);
\draw[ultra thin] (14,0) -- (14,3);
\draw[ultra thin] (16,0) -- (16,3);
\draw[ultra thin] (18,0) -- (18,3);
\draw[ultra thin] (20,0) -- (20,3);
\draw[thick, red, edge2] (0,0.5) -- (2,1.5);
\draw[thick, dash pattern=on 3pt off 3pt , red, edge2] (0,0.5) -- (2,0.5);
\draw[thick, red, edge2] (0,1.5) -- (2,1.5);
\draw[thick, red, edge2] (0,2.5) -- (2,1.5);
\draw[thick, dash pattern=on 3pt off 3pt , red, edge2] (0,2.5) -- (2,2.5);
\draw[thick, red, edge2] (2,0.5) -- (4,1.5);
\draw[thick, dash pattern=on 3pt off 3pt , red, edge2] (2,0.5) -- (4,0.5);
\draw[thick, red, edge2] (2,1.5) -- (4,1.5);
\draw[thick, red, edge2] (2,2.5) -- (4,1.5);
\draw[thick, dash pattern=on 3pt off 3pt , red, edge2] (2,2.5) -- (4,2.5);
\draw[thick, red, edge2] (4,0.5) -- (6,1.5);
\draw[thick, dash pattern=on 3pt off 3pt , red, edge2] (4,0.5) -- (6,0.5);
\draw[thick, red, edge2] (4,1.5) -- (6,1.5);
\draw[thick, red, edge2] (4,2.5) -- (6,1.5);
\draw[thick, dash pattern=on 3pt off 3pt , red, edge2] (4,2.5) -- (6,2.5);
\draw[thick, red, edge2] (6,0.5) -- (8,1.5);
\draw[thick, dash pattern=on 3pt off 3pt , red, edge2] (6,0.5) -- (8,0.5);
\draw[thick, red, edge2] (6,1.5) -- (8,1.5);
\draw[thick, red, edge2] (6,2.5) -- (8,1.5);
\draw[thick, dash pattern=on 3pt off 3pt , red, edge2] (6,2.5) -- (8,2.5);
\draw[thick, red, edge2] (8,0.5) -- (10,1.5);
\draw[thick, dash pattern=on 3pt off 3pt , red, edge2] (8,0.5) -- (10,0.5);
\draw[thick, red, edge2] (8,1.5) -- (10,1.5);
\draw[thick, red, edge2] (8,2.5) -- (10,1.5);
\draw[thick, dash pattern=on 3pt off 3pt , red, edge2] (8,2.5) -- (10,2.5);
\draw[thick, red, edge2] (10,0.5) -- (12,1.5);
\draw[thick, dash pattern=on 3pt off 3pt , red, edge2] (10,0.5) -- (12,0.5);
\draw[thick, red, edge2] (10,1.5) -- (12,1.5);
\draw[thick, red, edge2] (10,2.5) -- (12,1.5);
\draw[thick, dash pattern=on 3pt off 3pt , red, edge2] (10,2.5) -- (12,2.5);
\draw[thick, red, edge2] (12,0.5) -- (14,1.5);
\draw[thick, red, edge2] (12,1.5) -- (14,1.5);
\draw[thick, red, edge2] (12,2.5) -- (14,1.5);
\draw[thick, dash pattern=on 3pt off 3pt , red, edge2] (12,2.5) -- (14,2.5);
\draw[thick, red, edge2] (14,0.5) -- (16,0.5);
\draw[thick, red, edge2] (14,1.5) -- (16,1.5);
\draw[thick, red, edge2] (14,2.5) -- (16,1.5);
\draw[thick, dash pattern=on 3pt off 3pt , red, edge2] (14,2.5) -- (16,2.5);
\draw[thick, red, edge2] (16,0.5) -- (18,0.5);
\draw[thick, red, edge2] (16,1.5) -- (18,0.5);
\draw[thick, dash pattern=on 3pt off 3pt , red, edge2] (16,1.5) -- (18,1.5);
\draw[thick, red, edge2] (16,2.5) -- (18,1.5);
\draw[thick, dash pattern=on 3pt off 3pt , red, edge2] (16,2.5) -- (18,2.5);
\draw[thick, red, edge2] (18,0.5) -- (20,0.5);
\draw[thick, red, edge2] (18,1.5) -- (20,0.5);
\draw[thick, dash pattern=on 3pt off 3pt , red, edge2] (18,1.5) -- (20,1.5);
\draw[thick, red, edge2] (18,2.5) -- (20,1.5);
\draw[thick, dash pattern=on 3pt off 3pt , red, edge2] (18,2.5) -- (20,2.5);
\draw[thick, red, edge2] (12,-0.5) -- (14,-0.5);
\draw[thick, red, edge2] (14,-0.5) -- (16,0.5);
\end{tikzpicture} \\
{\footnotesize(e) $c_\textrm{lc} = 5$, $c_\textrm{merge}=150$, $\alpha=0.01$, $c_\textrm{left} = 0.1$}
\end{tabular}
\caption{Routing policy on a three-lane highway around a merge produced by our algorithm for various parameter values. The optimal actions are visualized as in Fig.\ \ref{fig:actions}. Each cell is 10m long. The figure shows a subset of the cells making up the highway; the goal cell is in the rightmost lane 5km down the road from the cells shown.}
\label{fig:highwaymerge}
\end{figure*}
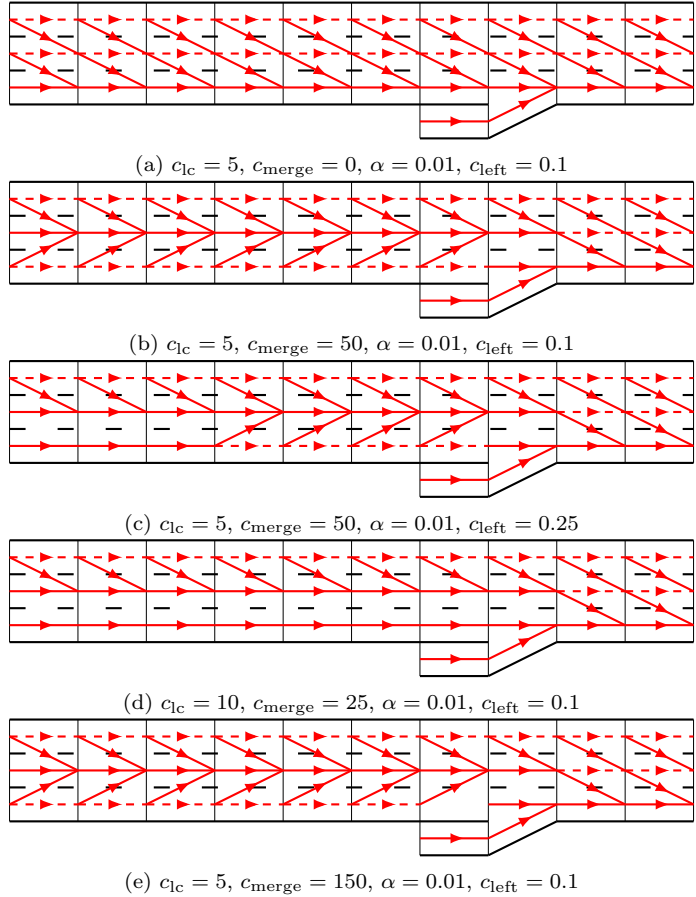

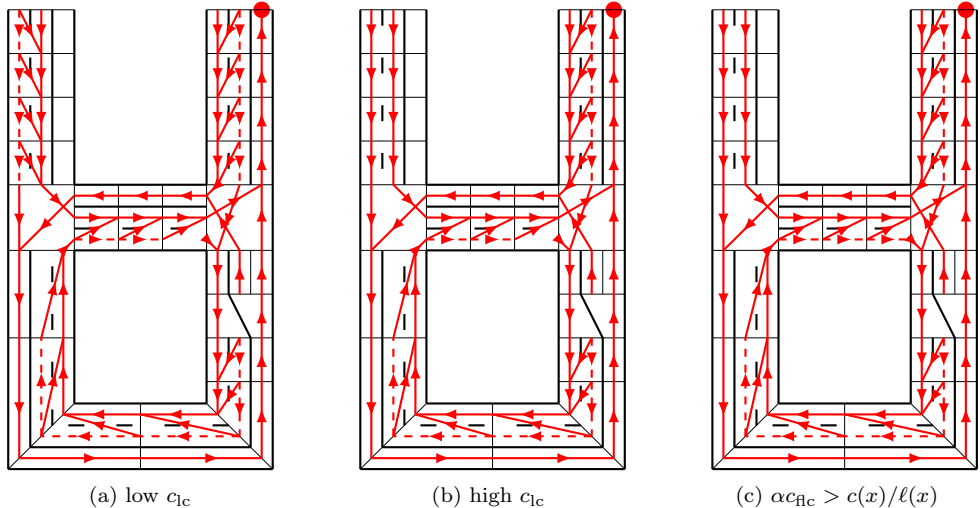
\begin{figure*}[t]
\centering
\begin{tabular}{ccc}
\begin{tikzpicture}[scale=0.29]
\filldraw[red] (11.5, 0) circle (10pt);
\draw[thick] (0, -0) -- (0, -21);
\draw[thick] (0, -21) -- (12, -21);
\draw[thick] (12, -21) -- (12, -0);
\draw[thick] (3, -0) -- (3, -8);
\draw[thick] (3, -8) -- (9, -8);
\draw[thick] (9, -8) -- (9, -0);
\draw[thick] (3, -11) -- (3, -18);
\draw[thick] (3, -18) -- (9, -18);
\draw[thick] (9, -18) -- (9, -11);
\draw[thick] (9, -11) -- (3, -11);
\draw[thick] (2, -0) -- (2, -8);
\draw[thick, dash pattern=on 6pt off 12pt] (1, -0) -- (1, -8);
\draw[thick] (11, -0) -- (11, -8);
\draw[thick, dash pattern=on 6pt off 12pt] (10, -0) -- (10, -8);
\draw[thick] (11, -20) -- (1, -20);
\draw[thick] (1, -20) -- (1, -11);
\draw[thick, dash pattern=on 6pt off 12pt] (10, -19) -- (2, -19);
\draw[thick, dash pattern=on 6pt off 12pt] (2, -19) -- (2, -11);
\draw[thick, dash pattern=on 6pt off 12pt] (3, -10) -- (9, -10);
\draw[thick] (3, -9) -- (9, -9);
\draw[ultra thin] (11, -11) -- (11, -13);
\draw[thick] (10, -11) -- (10, -13);
\draw[thick] (10, -13) -- (11, -15);
\draw[thick] (11, -15) -- (11, -20);
\draw[thick, dash pattern=on 6pt off 12pt] (10, -15) -- (10, -19);
\draw[ultra thin] (0, -0) -- (3, -0);
\draw[ultra thin] (0, -2) -- (3, -2);
\draw[ultra thin] (0, -4) -- (3, -4);
\draw[ultra thin] (0, -6) -- (3, -6);
\draw[ultra thin] (0, -8) -- (3, -8);
\draw[ultra thin] (9, -0) -- (12, -0);
\draw[ultra thin] (9, -2) -- (12, -2);
\draw[ultra thin] (9, -4) -- (12, -4);
\draw[ultra thin] (9, -6) -- (12, -6);
\draw[ultra thin] (9, -8) -- (12, -8);
\draw[ultra thin] (9, -18) -- (12, -21);
\draw[ultra thin] (6, -18) -- (6, -21);
\draw[ultra thin] (3, -18) -- (0, -21);
\draw[ultra thin] (3, -15) -- (0, -15);
\draw[ultra thin] (3, -11) -- (0, -11);
\draw[ultra thin] (3, -8) -- (3, -11);
\draw[ultra thin] (5, -8) -- (5, -11);
\draw[ultra thin] (7, -8) -- (7, -11);
\draw[ultra thin] (9, -8) -- (9, -11);
\draw[ultra thin] (9, -11) -- (12, -11);
\draw[ultra thin] (9, -13) -- (12, -13);
\draw[ultra thin] (9, -15) -- (12, -15);
\draw[ultra thin] (9, -17) -- (12, -17);
\draw[thick, red, edge2] (0.5, -0) -- (1.5, -2);
\draw[thick, dash pattern=on 3pt off 3pt , red, edge2] (0.5, -0) -- (0.5, -2);
\draw[thick, red, edge2] (1.5, -0) -- (1.5, -2);
\draw[thick, red, edge2] (0.5, -2) -- (1.5, -4);
\draw[thick, dash pattern=on 3pt off 3pt , red, edge2] (0.5, -2) -- (0.5, -4);
\draw[thick, red, edge2] (1.5, -2) -- (1.5, -4);
\draw[thick, red, edge2] (0.5, -4) -- (1.5, -6);
\draw[thick, dash pattern=on 3pt off 3pt , red, edge2] (0.5, -4) -- (0.5, -6);
\draw[thick, red, edge2] (1.5, -4) -- (1.5, -6);
\draw[thick, red, edge2] (0.5, -6) -- (1.5, -8);
\draw[thick, dash pattern=on 3pt off 3pt , red, edge2] (0.5, -6) -- (0.5, -8);
\draw[thick, red, edge2] (1.5, -6) -- (1.5, -8);
\draw[thick, red, edge2] (0.5, -8) -- (0.5, -11);
\draw[thick, red, edge2] (1.5, -8) -- (3, -9.5);
\draw[thick, red, edge2] (9.5, -0) -- (9.5, -2);
\draw[thick, red, edge2] (10.5, -0) -- (9.5, -2);
\draw[thick, dash pattern=on 3pt off 3pt , red, edge2] (10.5, -0) -- (10.5, -2);
\draw[thick, red, edge2] (9.5, -2) -- (9.5, -4);
\draw[thick, red, edge2] (10.5, -2) -- (9.5, -4);
\draw[thick, dash pattern=on 3pt off 3pt , red, edge2] (10.5, -2) -- (10.5, -4);
\draw[thick, red, edge2] (11.5, -2) -- (11.5, -0);
\draw[thick, red, edge2] (9.5, -4) -- (9.5, -6);
\draw[thick, red, edge2] (10.5, -4) -- (9.5, -6);
\draw[thick, dash pattern=on 3pt off 3pt , red, edge2] (10.5, -4) -- (10.5, -6);
\draw[thick, red, edge2] (11.5, -4) -- (11.5, -2);
\draw[thick, red, edge2] (9.5, -6) -- (9.5, -8);
\draw[thick, red, edge2] (10.5, -6) -- (9.5, -8);
\draw[thick, dash pattern=on 3pt off 3pt , red, edge2] (10.5, -6) -- (10.5, -8);
\draw[thick, red, edge2] (11.5, -6) -- (11.5, -4);
\draw[thick, red, edge2] (9.5, -8) -- (9, -8.5);
\draw[thick, red, edge2] (10.5, -8) -- (9.5, -11);
\draw[thick, red, edge2] (11.5, -8) -- (11.5, -6);
\draw[thick, red, edge2] (9.5, -18.5) -- (6, -18.5);
\draw[thick, red, edge2] (10.5, -19.5) -- (6, -18.5);
\draw[thick, dash pattern=on 3pt off 3pt , red, edge2] (10.5, -19.5) -- (6, -19.5);
\draw[thick, red, edge2] (11.5, -20.5) -- (11.5, -17);
\draw[thick, red, edge2] (6, -18.5) -- (2.5, -18.5);
\draw[thick, red, edge2] (6, -19.5) -- (2.5, -18.5);
\draw[thick, dash pattern=on 3pt off 3pt , red, edge2] (6, -19.5) -- (1.5, -19.5);
\draw[thick, red, edge2] (6, -20.5) -- (11.5, -20.5);
\draw[thick, red, edge2] (2.5, -18.5) -- (2.5, -15);
\draw[thick, red, edge2] (1.5, -19.5) -- (2.5, -15);
\draw[thick, dash pattern=on 3pt off 3pt , red, edge2] (1.5, -19.5) -- (1.5, -15);
\draw[thick, red, edge2] (0.5, -20.5) -- (6, -20.5);
\draw[thick, red, edge2] (2.5, -15) -- (2.5, -11);
\draw[thick, red, edge2] (1.5, -15) -- (2.5, -11);
\draw[thick, red, edge2] (0.5, -15) -- (0.5, -20.5);
\draw[thick, red, edge2] (2.5, -11) -- (3, -10.5);
\draw[thick, red, edge2] (0.5, -11) -- (0.5, -15);
\draw[thick, red, edge2] (3, -10.5) -- (5, -9.5);
\draw[thick, dash pattern=on 3pt off 3pt , red, edge2] (3, -10.5) -- (5, -10.5);
\draw[thick, red, edge2] (3, -9.5) -- (5, -9.5);
\draw[thick, red, edge2] (3, -8.5) -- (0.5, -11);
\draw[thick, red, edge2] (5, -10.5) -- (7, -9.5);
\draw[thick, dash pattern=on 3pt off 3pt , red, edge2] (5, -10.5) -- (7, -10.5);
\draw[thick, red, edge2] (5, -9.5) -- (7, -9.5);
\draw[thick, red, edge2] (5, -8.5) -- (3, -8.5);
\draw[thick, red, edge2] (7, -10.5) -- (9, -9.5);
\draw[thick, red, edge2] (7, -9.5) -- (9, -9.5);
\draw[thick, red, edge2] (7, -8.5) -- (5, -8.5);
\draw[thick, red, edge2] (9, -10.5) -- (9.5, -11);
\draw[thick, red, edge2] (9, -9.5) -- (11.5, -8);
\draw[thick, red, edge2] (9, -8.5) -- (7, -8.5);
\draw[thick, red, edge2] (9.5, -11) -- (9.5, -13);
\draw[thick, red, edge2] (9.5, -13) -- (9.5, -15);
\draw[thick, red, edge2] (10.5, -15) -- (9.5, -17);
\draw[thick, dash pattern=on 3pt off 3pt , red, edge2] (10.5, -15) -- (10.5, -17);
\draw[thick, red, edge2] (9.5, -15) -- (9.5, -17);
\draw[thick, red, edge2] (10.5, -17) -- (9.5, -18.5);
\draw[thick, dash pattern=on 3pt off 3pt , red, edge2] (10.5, -17) -- (10.5, -19.5);
\draw[thick, red, edge2] (9.5, -17) -- (9.5, -18.5);
\draw[thick, red, edge2] (11.5, -17) -- (11.5, -15);
\draw[thick, red, edge2] (11.5, -15) -- (11.5, -13);
\draw[thick, red, edge2] (10.5, -13) -- (10.5, -11);
\draw[thick, red, edge2] (11.5, -13) -- (11.5, -11);
\draw[thick, red, edge2] (10.5, -11) -- (9, -8.5);
\draw[thick, red, edge2] (11.5, -11) -- (11.5, -8);
\end{tikzpicture} &
\quad\quad \begin{tikzpicture}[scale=0.29]
\filldraw[red] (11.5, 0) circle (10pt);
\draw[thick] (0, -0) -- (0, -21);
\draw[thick] (0, -21) -- (12, -21);
\draw[thick] (12, -21) -- (12, -0);
\draw[thick] (3, -0) -- (3, -8);
\draw[thick] (3, -8) -- (9, -8);
\draw[thick] (9, -8) -- (9, -0);
\draw[thick] (3, -11) -- (3, -18);
\draw[thick] (3, -18) -- (9, -18);
\draw[thick] (9, -18) -- (9, -11);
\draw[thick] (9, -11) -- (3, -11);
\draw[thick] (2, -0) -- (2, -8);
\draw[thick, dash pattern=on 6pt off 12pt] (1, -0) -- (1, -8);
\draw[thick] (11, -0) -- (11, -8);
\draw[thick, dash pattern=on 6pt off 12pt] (10, -0) -- (10, -8);
\draw[thick] (11, -20) -- (1, -20);
\draw[thick] (1, -20) -- (1, -11);
\draw[thick, dash pattern=on 6pt off 12pt] (10, -19) -- (2, -19);
\draw[thick, dash pattern=on 6pt off 12pt] (2, -19) -- (2, -11);
\draw[thick, dash pattern=on 6pt off 12pt] (3, -10) -- (9, -10);
\draw[thick] (3, -9) -- (9, -9);
\draw[ultra thin] (11, -11) -- (11, -13);
\draw[thick] (10, -11) -- (10, -13);
\draw[thick] (10, -13) -- (11, -15);
\draw[thick] (11, -15) -- (11, -20);
\draw[thick, dash pattern=on 6pt off 12pt] (10, -15) -- (10, -19);
\draw[ultra thin] (0, -0) -- (3, -0);
\draw[ultra thin] (0, -2) -- (3, -2);
\draw[ultra thin] (0, -4) -- (3, -4);
\draw[ultra thin] (0, -6) -- (3, -6);
\draw[ultra thin] (0, -8) -- (3, -8);
\draw[ultra thin] (9, -0) -- (12, -0);
\draw[ultra thin] (9, -2) -- (12, -2);
\draw[ultra thin] (9, -4) -- (12, -4);
\draw[ultra thin] (9, -6) -- (12, -6);
\draw[ultra thin] (9, -8) -- (12, -8);
\draw[ultra thin] (9, -18) -- (12, -21);
\draw[ultra thin] (6, -18) -- (6, -21);
\draw[ultra thin] (3, -18) -- (0, -21);
\draw[ultra thin] (3, -15) -- (0, -15);
\draw[ultra thin] (3, -11) -- (0, -11);
\draw[ultra thin] (3, -8) -- (3, -11);
\draw[ultra thin] (5, -8) -- (5, -11);
\draw[ultra thin] (7, -8) -- (7, -11);
\draw[ultra thin] (9, -8) -- (9, -11);
\draw[ultra thin] (9, -11) -- (12, -11);
\draw[ultra thin] (9, -13) -- (12, -13);
\draw[ultra thin] (9, -15) -- (12, -15);
\draw[ultra thin] (9, -17) -- (12, -17);
\draw[thick, red, edge2] (0.5, -0) -- (0.5, -2);
\draw[thick, red, edge2] (1.5, -0) -- (1.5, -2);
\draw[thick, red, edge2] (0.5, -2) -- (0.5, -4);
\draw[thick, red, edge2] (1.5, -2) -- (1.5, -4);
\draw[thick, red, edge2] (0.5, -4) -- (0.5, -6);
\draw[thick, red, edge2] (1.5, -4) -- (1.5, -6);
\draw[thick, red, edge2] (0.5, -6) -- (0.5, -8);
\draw[thick, red, edge2] (1.5, -6) -- (1.5, -8);
\draw[thick, red, edge2] (0.5, -8) -- (0.5, -11);
\draw[thick, red, edge2] (1.5, -8) -- (3, -9.5);
\draw[thick, red, edge2] (9.5, -0) -- (9.5, -2);
\draw[thick, red, edge2] (10.5, -0) -- (9.5, -2);
\draw[thick, dash pattern=on 3pt off 3pt , red, edge2] (10.5, -0) -- (10.5, -2);
\draw[thick, red, edge2] (9.5, -2) -- (9.5, -4);
\draw[thick, red, edge2] (10.5, -2) -- (9.5, -4);
\draw[thick, dash pattern=on 3pt off 3pt , red, edge2] (10.5, -2) -- (10.5, -4);
\draw[thick, red, edge2] (11.5, -2) -- (11.5, -0);
\draw[thick, red, edge2] (9.5, -4) -- (9.5, -6);
\draw[thick, red, edge2] (10.5, -4) -- (9.5, -6);
\draw[thick, dash pattern=on 3pt off 3pt , red, edge2] (10.5, -4) -- (10.5, -6);
\draw[thick, red, edge2] (11.5, -4) -- (11.5, -2);
\draw[thick, red, edge2] (9.5, -6) -- (9.5, -8);
\draw[thick, red, edge2] (10.5, -6) -- (9.5, -8);
\draw[thick, dash pattern=on 3pt off 3pt , red, edge2] (10.5, -6) -- (10.5, -8);
\draw[thick, red, edge2] (11.5, -6) -- (11.5, -4);
\draw[thick, red, edge2] (9.5, -8) -- (9, -8.5);
\draw[thick, red, edge2] (10.5, -8) -- (9.5, -11);
\draw[thick, red, edge2] (11.5, -8) -- (11.5, -6);
\draw[thick, red, edge2] (9.5, -18.5) -- (6, -18.5);
\draw[thick, red, edge2] (10.5, -19.5) -- (6, -18.5);
\draw[thick, dash pattern=on 3pt off 3pt , red, edge2] (10.5, -19.5) -- (6, -19.5);
\draw[thick, red, edge2] (11.5, -20.5) -- (11.5, -17);
\draw[thick, red, edge2] (6, -18.5) -- (2.5, -18.5);
\draw[thick, red, edge2] (6, -19.5) -- (2.5, -18.5);
\draw[thick, dash pattern=on 3pt off 3pt , red, edge2] (6, -19.5) -- (1.5, -19.5);
\draw[thick, red, edge2] (6, -20.5) -- (11.5, -20.5);
\draw[thick, red, edge2] (2.5, -18.5) -- (2.5, -15);
\draw[thick, red, edge2] (1.5, -19.5) -- (2.5, -15);
\draw[thick, dash pattern=on 3pt off 3pt , red, edge2] (1.5, -19.5) -- (1.5, -15);
\draw[thick, red, edge2] (0.5, -20.5) -- (6, -20.5);
\draw[thick, red, edge2] (2.5, -15) -- (2.5, -11);
\draw[thick, red, edge2] (1.5, -15) -- (2.5, -11);
\draw[thick, red, edge2] (0.5, -15) -- (0.5, -20.5);
\draw[thick, red, edge2] (2.5, -11) -- (3, -10.5);
\draw[thick, red, edge2] (0.5, -11) -- (0.5, -15);
\draw[thick, red, edge2] (3, -10.5) -- (5, -9.5);
\draw[thick, dash pattern=on 3pt off 3pt , red, edge2] (3, -10.5) -- (5, -10.5);
\draw[thick, red, edge2] (3, -9.5) -- (5, -9.5);
\draw[thick, red, edge2] (3, -8.5) -- (0.5, -11);
\draw[thick, red, edge2] (5, -10.5) -- (7, -9.5);
\draw[thick, dash pattern=on 3pt off 3pt , red, edge2] (5, -10.5) -- (7, -10.5);
\draw[thick, red, edge2] (5, -9.5) -- (7, -9.5);
\draw[thick, red, edge2] (5, -8.5) -- (3, -8.5);
\draw[thick, red, edge2] (7, -10.5) -- (9, -9.5);
\draw[thick, red, edge2] (7, -9.5) -- (9, -9.5);
\draw[thick, red, edge2] (7, -8.5) -- (5, -8.5);
\draw[thick, red, edge2] (9, -10.5) -- (9.5, -11);
\draw[thick, red, edge2] (9, -9.5) -- (11.5, -8);
\draw[thick, red, edge2] (9, -8.5) -- (7, -8.5);
\draw[thick, red, edge2] (9.5, -11) -- (9.5, -13);
\draw[thick, red, edge2] (9.5, -13) -- (9.5, -15);
\draw[thick, red, edge2] (10.5, -15) -- (9.5, -17);
\draw[thick, dash pattern=on 3pt off 3pt , red, edge2] (10.5, -15) -- (10.5, -17);
\draw[thick, red, edge2] (9.5, -15) -- (9.5, -17);
\draw[thick, red, edge2] (10.5, -17) -- (9.5, -18.5);
\draw[thick, dash pattern=on 3pt off 3pt , red, edge2] (10.5, -17) -- (10.5, -19.5);
\draw[thick, red, edge2] (9.5, -17) -- (9.5, -18.5);
\draw[thick, red, edge2] (11.5, -17) -- (11.5, -15);
\draw[thick, red, edge2] (11.5, -15) -- (11.5, -13);
\draw[thick, red, edge2] (10.5, -13) -- (10.5, -11);
\draw[thick, red, edge2] (11.5, -13) -- (11.5, -11);
\draw[thick, red, edge2] (10.5, -11) -- (9, -8.5);
\draw[thick, red, edge2] (11.5, -11) -- (11.5, -8);
\end{tikzpicture} &
\quad\quad \begin{tikzpicture}[scale=0.29]
\filldraw[red] (11.5, 0) circle (10pt);
\draw[thick] (0, -0) -- (0, -21);
\draw[thick] (0, -21) -- (12, -21);
\draw[thick] (12, -21) -- (12, -0);
\draw[thick] (3, -0) -- (3, -8);
\draw[thick] (3, -8) -- (9, -8);
\draw[thick] (9, -8) -- (9, -0);
\draw[thick] (3, -11) -- (3, -18);
\draw[thick] (3, -18) -- (9, -18);
\draw[thick] (9, -18) -- (9, -11);
\draw[thick] (9, -11) -- (3, -11);
\draw[thick] (2, -0) -- (2, -8);
\draw[thick, dash pattern=on 6pt off 12pt] (1, -0) -- (1, -8);
\draw[thick] (11, -0) -- (11, -8);
\draw[thick, dash pattern=on 6pt off 12pt] (10, -0) -- (10, -8);
\draw[thick] (11, -20) -- (1, -20);
\draw[thick] (1, -20) -- (1, -11);
\draw[thick, dash pattern=on 6pt off 12pt] (10, -19) -- (2, -19);
\draw[thick, dash pattern=on 6pt off 12pt] (2, -19) -- (2, -11);
\draw[thick, dash pattern=on 6pt off 12pt] (3, -10) -- (9, -10);
\draw[thick] (3, -9) -- (9, -9);
\draw[ultra thin] (11, -11) -- (11, -13);
\draw[thick] (10, -11) -- (10, -13);
\draw[thick] (10, -13) -- (11, -15);
\draw[thick] (11, -15) -- (11, -20);
\draw[thick, dash pattern=on 6pt off 12pt] (10, -15) -- (10, -19);
\draw[ultra thin] (0, -0) -- (3, -0);
\draw[ultra thin] (0, -2) -- (3, -2);
\draw[ultra thin] (0, -4) -- (3, -4);
\draw[ultra thin] (0, -6) -- (3, -6);
\draw[ultra thin] (0, -8) -- (3, -8);
\draw[ultra thin] (9, -0) -- (12, -0);
\draw[ultra thin] (9, -2) -- (12, -2);
\draw[ultra thin] (9, -4) -- (12, -4);
\draw[ultra thin] (9, -6) -- (12, -6);
\draw[ultra thin] (9, -8) -- (12, -8);
\draw[ultra thin] (9, -18) -- (12, -21);
\draw[ultra thin] (6, -18) -- (6, -21);
\draw[ultra thin] (3, -18) -- (0, -21);
\draw[ultra thin] (3, -15) -- (0, -15);
\draw[ultra thin] (3, -11) -- (0, -11);
\draw[ultra thin] (3, -8) -- (3, -11);
\draw[ultra thin] (5, -8) -- (5, -11);
\draw[ultra thin] (7, -8) -- (7, -11);
\draw[ultra thin] (9, -8) -- (9, -11);
\draw[ultra thin] (9, -11) -- (12, -11);
\draw[ultra thin] (9, -13) -- (12, -13);
\draw[ultra thin] (9, -15) -- (12, -15);
\draw[ultra thin] (9, -17) -- (12, -17);
\draw[thick, red, edge2] (0.5, -0) -- (0.5, -2);
\draw[thick, red, edge2] (1.5, -0) -- (1.5, -2);
\draw[thick, red, edge2] (0.5, -2) -- (0.5, -4);
\draw[thick, red, edge2] (1.5, -2) -- (1.5, -4);
\draw[thick, red, edge2] (0.5, -4) -- (0.5, -6);
\draw[thick, red, edge2] (1.5, -4) -- (1.5, -6);
\draw[thick, red, edge2] (0.5, -6) -- (0.5, -8);
\draw[thick, red, edge2] (1.5, -6) -- (1.5, -8);
\draw[thick, red, edge2] (0.5, -8) -- (0.5, -11);
\draw[thick, red, edge2] (1.5, -8) -- (3, -9.5);
\draw[thick, red, edge2] (9.5, -0) -- (9.5, -2);
\draw[thick, red, edge2] (10.5, -0) -- (9.5, -2);
\draw[thick, dash pattern=on 3pt off 3pt , red, edge2] (10.5, -0) -- (10.5, -2);
\draw[thick, red, edge2] (9.5, -2) -- (9.5, -4);
\draw[thick, red, edge2] (10.5, -2) -- (9.5, -4);
\draw[thick, dash pattern=on 3pt off 3pt , red, edge2] (10.5, -2) -- (10.5, -4);
\draw[thick, red, edge2] (11.5, -2) -- (11.5, -0);
\draw[thick, red, edge2] (9.5, -4) -- (9.5, -6);
\draw[thick, red, edge2] (10.5, -4) -- (9.5, -6);
\draw[thick, dash pattern=on 3pt off 3pt , red, edge2] (10.5, -4) -- (10.5, -6);
\draw[thick, red, edge2] (11.5, -4) -- (11.5, -2);
\draw[thick, red, edge2] (9.5, -6) -- (9.5, -8);
\draw[thick, red, edge2] (10.5, -6) -- (9.5, -8);
\draw[thick, dash pattern=on 3pt off 3pt , red, edge2] (10.5, -6) -- (10.5, -8);
\draw[thick, red, edge2] (11.5, -6) -- (11.5, -4);
\draw[thick, red, edge2] (9.5, -8) -- (9, -8.5);
\draw[thick, red, edge2] (10.5, -8) -- (9.5, -11);
\draw[thick, red, edge2] (11.5, -8) -- (11.5, -6);
\draw[thick, red, edge2] (9.5, -18.5) -- (6, -18.5);
\draw[thick, red, edge2] (10.5, -19.5) -- (6, -18.5);
\draw[thick, dash pattern=on 3pt off 3pt , red, edge2] (10.5, -19.5) -- (6, -19.5);
\draw[thick, red, edge2] (11.5, -20.5) -- (11.5, -17);
\draw[thick, red, edge2] (6, -18.5) -- (2.5, -18.5);
\draw[thick, red, edge2] (6, -19.5) -- (2.5, -18.5);
\draw[thick, dash pattern=on 3pt off 3pt , red, edge2] (6, -19.5) -- (1.5, -19.5);
\draw[thick, red, edge2] (6, -20.5) -- (11.5, -20.5);
\draw[thick, red, edge2] (2.5, -18.5) -- (2.5, -15);
\draw[thick, red, edge2] (1.5, -19.5) -- (2.5, -15);
\draw[thick, dash pattern=on 3pt off 3pt , red, edge2] (1.5, -19.5) -- (1.5, -15);
\draw[thick, red, edge2] (0.5, -20.5) -- (6, -20.5);
\draw[thick, red, edge2] (2.5, -15) -- (2.5, -11);
\draw[thick, red, edge2] (1.5, -15) -- (2.5, -11);
\draw[thick, red, edge2] (0.5, -15) -- (0.5, -20.5);
\draw[thick, red, edge2] (2.5, -11) -- (3, -10.5);
\draw[thick, red, edge2] (0.5, -11) -- (0.5, -15);
\draw[thick, red, edge2] (3, -10.5) -- (5, -9.5);
\draw[thick, dash pattern=on 3pt off 3pt , red, edge2] (3, -10.5) -- (5, -10.5);
\draw[thick, red, edge2] (3, -9.5) -- (5, -9.5);
\draw[thick, red, edge2] (3, -8.5) -- (0.5, -11);
\draw[thick, red, edge2] (5, -10.5) -- (7, -9.5);
\draw[thick, dash pattern=on 3pt off 3pt , red, edge2] (5, -10.5) -- (7, -10.5);
\draw[thick, red, edge2] (5, -9.5) -- (7, -9.5);
\draw[thick, red, edge2] (5, -8.5) -- (3, -8.5);
\draw[thick, red, edge2] (7, -10.5) -- (9, -9.5);
\draw[thick, dash pattern=on 3pt off 3pt , red, edge2] (7, -10.5) -- (9, -10.5);
\draw[thick, red, edge2] (7, -9.5) -- (9, -9.5);
\draw[thick, red, edge2] (7, -8.5) -- (5, -8.5);
\draw[thick, red, edge2] (9, -10.5) -- (9.5, -11);
\draw[thick, red, edge2] (9, -9.5) -- (11.5, -8);
\draw[thick, red, edge2] (9, -8.5) -- (7, -8.5);
\draw[thick, red, edge2] (9.5, -11) -- (9.5, -13);
\draw[thick, red, edge2] (9.5, -13) -- (9.5, -15);
\draw[thick, red, edge2] (10.5, -15) -- (9.5, -17);
\draw[thick, dash pattern=on 3pt off 3pt , red, edge2] (10.5, -15) -- (10.5, -17);
\draw[thick, red, edge2] (9.5, -15) -- (9.5, -17);
\draw[thick, red, edge2] (10.5, -17) -- (9.5, -18.5);
\draw[thick, dash pattern=on 3pt off 3pt , red, edge2] (10.5, -17) -- (10.5, -19.5);
\draw[thick, red, edge2] (9.5, -17) -- (9.5, -18.5);
\draw[thick, red, edge2] (11.5, -17) -- (11.5, -15);
\draw[thick, red, edge2] (11.5, -15) -- (11.5, -13);
\draw[thick, red, edge2] (10.5, -13) -- (10.5, -11);
\draw[thick, red, edge2] (11.5, -13) -- (11.5, -11);
\draw[thick, red, edge2] (10.5, -11) -- (9, -8.5);
\draw[thick, red, edge2] (11.5, -11) -- (11.5, -8);
\end{tikzpicture}\\
{\footnotesize (a) low $c_\lc$} &
{\footnotesize \quad\quad (b) high $c_\lc$} &
{\footnotesize \quad\quad (c) $\alpha c_\flc > c(x) / \ell(x)$ } \\
\end{tabular}
\caption{Routing policy on a more complex road network produced by our algorithm for two different values of lane change cost, as well as a forced lane change cost that violates the monotonicity requirement. The goal is in the top-right corner.}
\label{fig:intersections}
\end{figure*}

We implemented the presented router in our autonomous vehicle software at Ike for routing class-A trucks on highways (Ike was acquired by Nuro in 2021) and at Nuro to route delivery robots on surface streets. As it is challenging to effectively visualize the routing results on expansive road networks, we first illustrate the rich class of routing behavior our algorithm can generate in two constructed but representative scenarios for various parameter values. We then show some examples of the routing behavior on more compact real-world networks.

\subsection{Representative examples}

In our first experiment we consider a straight three-lane highway that has an on-ramp merging with the rightmost lane (see \figref{highwaymerge}). The goal cell is in the rightmost lane 5km down from the section of the highway shown in the figure. We show the spectrum of routing policies our algorithm can produce for varying parameter values. The length of each cell is 10m and the cost $c(x)$ of each cell $x$ is equal to its length times a factor $\sigma$ that penalizes not driving in the rightmost lane. We have $\sigma = 1 + mc_\mathrm{left}$, where $m = 0, 1, 2$ for the rightmost, middle, and leftmost lane respectively; $c_\mathrm{left}$ is a parameter that we vary in the experiments. In addition, we penalize being routed through the merge, so each cell with a successor with multiple predecessors gets an additional cost of $c_\mathrm{merge}$, which we vary in the experiments. In all experiments, we set $c_\flc = 1/\alpha$.

In the first parameter setting (see \figref{highwaymerge}(a)) we have no cost for going through a merge, and the resulting routing policy is to always try to lane change to the rightmost lane of the road. With a larger merge cost (see \figref{highwaymerge}(b)), we see that the router will try to lane change out of the right lane into the middle lane before the merge in order to avoid it, and then lane change back to the right lane after the merge. If we modestly increase the penalty for not driving in the rightmost lane (see \figref{highwaymerge}(c)) we see that the router policy waits with attempting to lane change out of the right lane until we are relatively close to the merge. Increasing the cost for a lane change and reducing the cost for the merge (see \figref{highwaymerge}(d)) leads to a policy where we do not attempt to move out of the right lane to avoid the merge; however, if one is already in the middle lane, we would not move to the right lane until after the merge. If we make the merge cost very large (see \figref{highwaymerge}(e)), the router tries to avoid the merge using a forced lane change.

In our second experiment we consider a somewhat more complex road network (see \figref{intersections}). The goal is in the top right corner. A vehicle starting out in the top left corner has two potential paths of reaching the goal. The first option is to perform a lane change into the left lane such that it can make a left turn in the intersection. It can then keep left to reach the goal with no further lane changes. The second option allows it to stay in the right lane. However it must then travel a longer path around the perimeter of this road map. For a low lane change cost our algorithm determines an optimal policy that prefers to change lanes and thus take the shorter of the two paths (see \figref{intersections}(a)). For a larger lane change cost however, the optimal policy takes the longer path in order to avoid  a lane change (see \figref{intersections}(b)). This illustrates the ability of our algorithm to make both macroscopic decisions (which roads to take to the goal) as well as the microscopic decisions of when to change lanes.

Finally, we consider the case of a forced lane change cost so large that the monotonicity requirement is violated. In this case optimal policies may contain cycles (see \figref{intersections}(c)), and the problem can no longer be solved with a single pass of Dijkstra's algorithm. Instead, the queue never empties as we keep encountering (and having to reopen) previously closed nodes. Eventually the value of these nodes converges, as we are effectively performing value iteration (policy iteration would be a less inefficient choice in this case). This example illustrates that the computational benefit of using a monotone cost formulation comes somewhat at the expense of the richness of routing policies that may result.

\subsection{Real road network examples}

\begin{figure*}[h]
\centering
\includegraphics[scale=0.4]{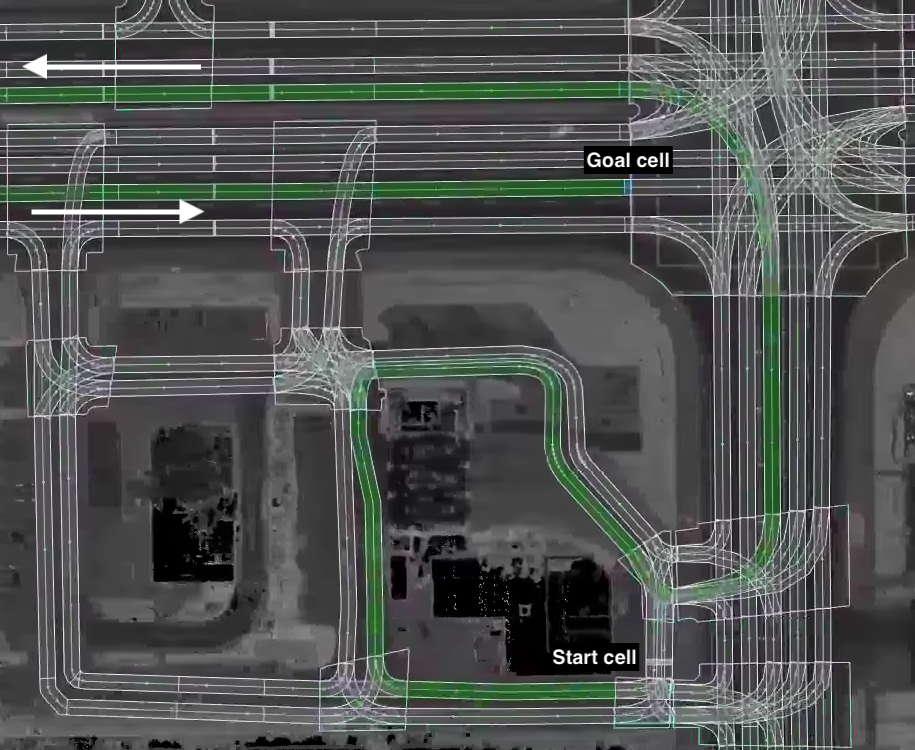}
\caption{The maximum probability path taken on a real-world road network, with $\alpha = 0.00128$. Due to the small value of $\alpha$, the policy on the most probable path does not even attempt a lane change as failing a lane change could be potentially costly. Instead the most probable path to reach the goal from the start cell involves a sequence of right turns, a left turn and a u-turn.}
\label{fig:long}
\end{figure*}

In the next collection of experiments, we highlight various routing policies computed in more complex road networks. As visualizing an entire policy on real-world road networks results in cluttered graphics we will make use of the maximum-likelihood single-path route representation described above. This representation is sufficient in order to convey the choices made by the routing algorithm for different parameter values. Specifically, we will visualize the maximum-likelihood path by shading only the cells the vehicle is most likely to traverse in its path to the goal cell. However, note that the policy for the real-world road networks still covers every cell in the network. We will also choose the shading scheme based on the optimal action in each individual cell: A cell shaded in green denotes a stay-in-lane action is optimal; a cell shaded in yellow denotes a cell in which a lane change (or forced lane change) is optimal with diagonal white lines showing the direction of the desired lane change.

In the first experiment we show how the magnitude of value of $\alpha$ can change the policy output drastically. A subset of the road network is shown in \figref{long}, where initially we start in the bottom left on a single-lane road. The goal cell is located in the upper-left (in the multi-lane road moving from left to right in the figure, the goal node is on the lane to the left of the right-most lane). In this experiment, we used the parameters $c_\textrm{lc} = 5$, $c_\textrm{merge}=5$, $\alpha=0.00128$, and $c_\textrm{left} = 0.25$. Due to the small value of $\alpha$ the optimal policy is to perform no lane changes at all. This results in a maximum-likelihood path which loops around the road network by performing several protected right, left, and u-turns. To traverse this route, the vehicle would need to move a total of 859 meters. Alternatively by computing a new policy with $\alpha=0.01$ (see \figref{cropped}), we obtain a much shorter route with a single lane change to the left. Here the route has a total length of 125 meters --- over 85\% shorter than the previous route with no lane changes. Just as the previous experiments showed, the value of $\alpha$ (among the other parameters) can change the behavior of the router policy significantly.

\begin{figure*}[h]
\centering
\includegraphics[scale=0.4]{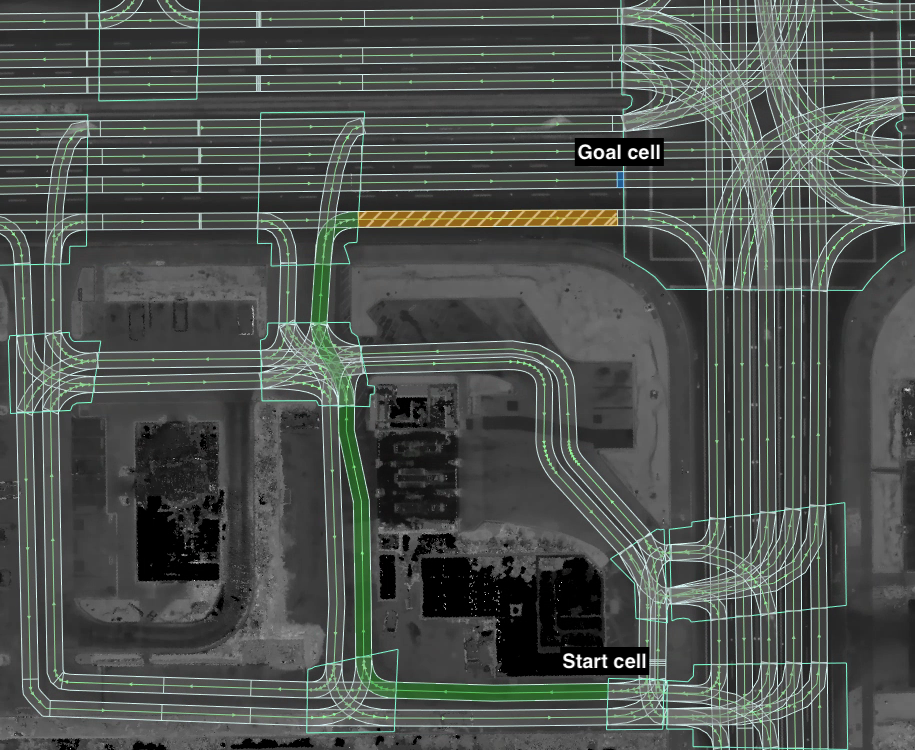}
\caption{The same road network as \figref{long}, with $\alpha = 0.01$. The maximum probability path is to perform two right turns followed by a single lane change (indicated by the yellow shaded cell). This path has total length much shorter than the path computed in \figref{long}.}
\label{fig:cropped}
\end{figure*}

\begin{figure*}[h]
\centering
\includegraphics[scale=0.4]{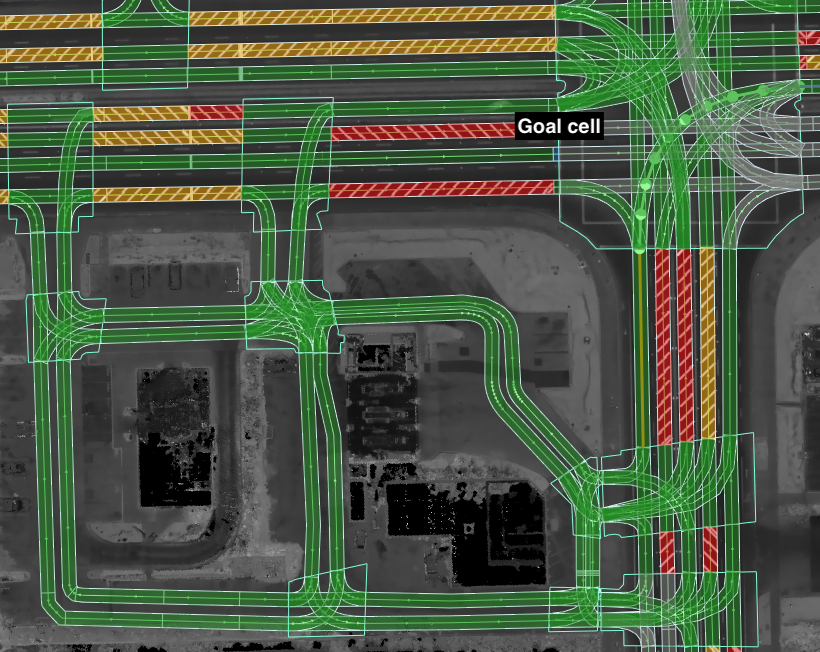}
\caption{The same road network as in \figref{long} and \figref{cropped}, with $\alpha = 0.01$. Here, we render the entire policy. A cell is shaded green if the optimal action is stay in lane, yellow if the optimal action is to perform a lane change (with the direction of the diagonal white lines indicating the direction of the lane change), and red if the optimal action is to perform a forced lane change (with the direction of the diagonal white lines indicating the direction of the forced lane change).}
\label{fig:real-monotone}
\end{figure*}

\begin{figure*}[h]
\centering
\includegraphics[scale=0.4]{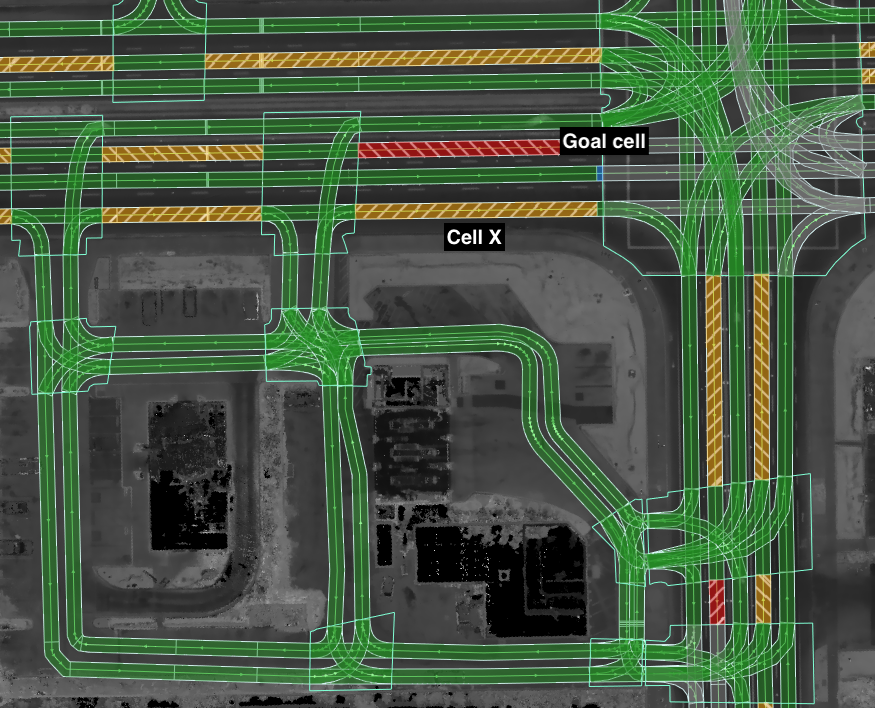}
\caption{The same road network as in \figref{real-monotone}, with the monotonicity condition violated (i.e., $\alpha c_\flc > c(x) / \ell(x)$), allowing cycles in the resulting policy. This a similar behavior to \figref{intersections}(c), in which the optimal action at cell $X$ is to perform a regular lane change action (not a forced lane change action, as in \figref{real-monotone} for the same cell). If the lane change action is not successful, we perform a loop in the network and attempt to reach the goal at a later point in time.}
\label{fig:real-no-monotone}
\end{figure*}

While this experiment demonstrates that much shorter routes can be taken in the presence of lane change actions, one still needs to be careful when tweaking the value of $\alpha$. One can imagine a network in which the optimal route is to take a path which requires turning right onto a multi-lane road and immediately performing a lane change into a left lane to take an upcoming left turn. On the other hand, it may be safer to take a slightly longer path (as demonstrated by \figref{long}) and have a longer stretch of road to perform any necessary lane changes, even if the cost of the resulting route is not optimal. This is one potential restriction of the policy generated by our proposed algorithm proposed, as the resulting policy can be quite different depending on the magnitude of $\alpha$. One can also view this shorter path computed in \figref{long} as a consequence of the monotonicity conditions we impose on the policy. Indeed, it is the shortest path from the start cell to the goal cell---\figref{real-monotone} shows the full policy computed over the network, but a slightly longer path with more room to perform a lane change may be preferable. Alternatively, by relaxing the monotonicity condition we obtain cycles in the generated policy\footnote{Note that with cycles in the policy, we no longer have a clean definition of the maximum-likelihood route.} which allows the same lane change action to attempted multiple times (rather than a single forced lane change action), see \figref{real-no-monotone}. This policy matches the cyclic policy illustrated in \figref{intersections}(c).

Next, we discuss the impact of having cells of varying length. In the above experiments, cells could have length between most 100 meters (any cell which had length more than 100 meters was subdivided). However, suppose that within this 100 meter horizon, a vehicle wants to lane change multiple times from the right-most to the left-most lane in order to take an upcoming left turn. If the distance from the vehicle's current position to the intersection is less than 100 meters, it would not be possible to perform more than one lane change action according to the policy. However by decreasing the maximum length of a cell down to 10 meters (possibly increasing the size of the resulting lane graph by an order of magnitude), we obtain a more granular policy with the potential to perform multiple lane change actions within a 100 meter window. See \figref{lane-changes}. This raises a natural trade-off in the size of the lane graph (by restricting the maximum length of a cell, thereby increasing the number of cells in the lane graph and thus the cost for computing a policy), and granularity of the resulting policy.

\begin{figure*}[h]
\centering
\includegraphics[scale=0.4]{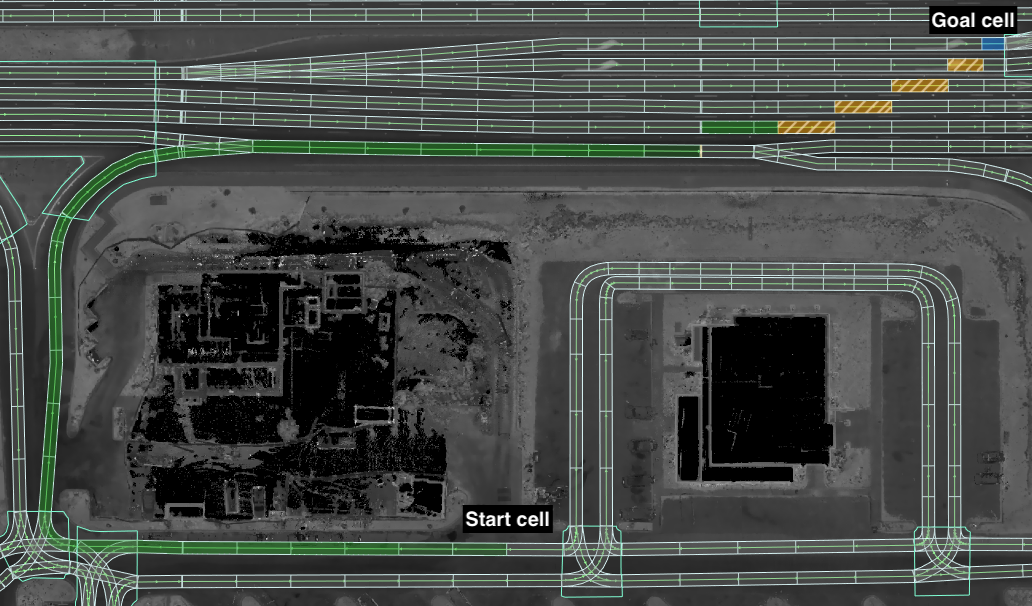}
\caption{A road-network in which the vehicle begins traversing a lane at the bottom of the network from right to left, and makes two right turns onto a multi-lane road where it must lane change into the left-hand lane for an upcoming left turn. Here, $\alpha = 0.01$. With shorter lane graph cells, the maximum probability path is to make the two right turns and then attempt to perform a sequence of lane changes as it approaches the intersection.}
\label{fig:lane-changes}
\end{figure*}

\subsection{Comparison against existing algorithms}

We now compare Algorithm~\ref{alg:router} against existing algorithms for solving MDPs. Specifically, we use the open source library AI-toolbox \citep{JMLR:v21:18-402} which contains C++ implementations of common solvers. Algorithm~\ref{alg:router} was implemented in C++, and all experiments were run on a 2.2 GHz Intel Xeon Silver 4114 CPU with 64 GB of RAM.

First, we compute an optimal policy over the same mock network as in \figref{intersections} using the value iteration algorithm (run for a sufficient number of times until convergence). See \figref{intersections-vi}. Upon closer inspection, it seems that the policies themselves are quite similar to the policies rendered in \figref{intersections}. For low $c_\lc$ (\figref{intersections-vi} (a)), the policy generated by value iteration chooses one additional forced lane change action over the policy generated by our algorithm. Similarly, one can tune $c_\lc$ by increasing it (\figref{intersections-vi} (b)) to prevent an abundance of lane change actions, and using a sufficiently large value of $c_\flc$ such that $\alpha c_\flc > c(x) / \ell(x)$ does indeed introduce cycles in the resulting policy (\figref{intersections-vi} (c)).

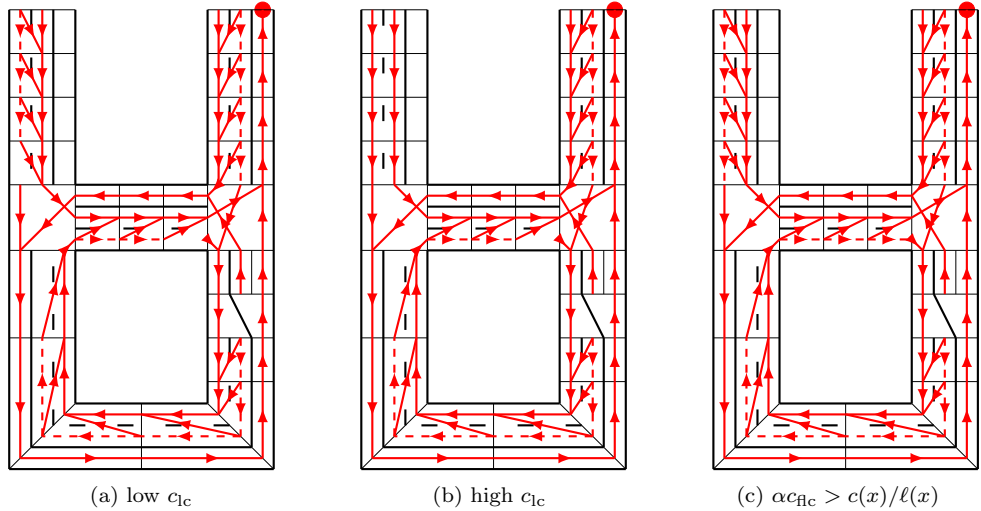
\begin{figure*}[t]
\centering
\begin{tabular}{ccc}
\begin{tikzpicture}[scale=0.29]
\filldraw[red] (11.5, 0) circle (10pt);
\draw[thick] (0, -0) -- (0, -21);
\draw[thick] (0, -21) -- (12, -21);
\draw[thick] (12, -21) -- (12, -0);
\draw[thick] (3, -0) -- (3, -8);
\draw[thick] (3, -8) -- (9, -8);
\draw[thick] (9, -8) -- (9, -0);
\draw[thick] (3, -11) -- (3, -18);
\draw[thick] (3, -18) -- (9, -18);
\draw[thick] (9, -18) -- (9, -11);
\draw[thick] (9, -11) -- (3, -11);
\draw[thick] (2, -0) -- (2, -8);
\draw[thick, dash pattern=on 6pt off 12pt] (1, -0) -- (1, -8);
\draw[thick] (11, -0) -- (11, -8);
\draw[thick, dash pattern=on 6pt off 12pt] (10, -0) -- (10, -8);
\draw[thick] (11, -20) -- (1, -20);
\draw[thick] (1, -20) -- (1, -11);
\draw[thick, dash pattern=on 6pt off 12pt] (10, -19) -- (2, -19);
\draw[thick, dash pattern=on 6pt off 12pt] (2, -19) -- (2, -11);
\draw[thick, dash pattern=on 6pt off 12pt] (3, -10) -- (9, -10);
\draw[thick] (3, -9) -- (9, -9);
\draw[ultra thin] (11, -11) -- (11, -13);
\draw[thick] (10, -11) -- (10, -13);
\draw[thick] (10, -13) -- (11, -15);
\draw[thick] (11, -15) -- (11, -20);
\draw[thick, dash pattern=on 6pt off 12pt] (10, -15) -- (10, -19);
\draw[ultra thin] (0, -0) -- (3, -0);
\draw[ultra thin] (0, -2) -- (3, -2);
\draw[ultra thin] (0, -4) -- (3, -4);
\draw[ultra thin] (0, -6) -- (3, -6);
\draw[ultra thin] (0, -8) -- (3, -8);
\draw[ultra thin] (9, -0) -- (12, -0);
\draw[ultra thin] (9, -2) -- (12, -2);
\draw[ultra thin] (9, -4) -- (12, -4);
\draw[ultra thin] (9, -6) -- (12, -6);
\draw[ultra thin] (9, -8) -- (12, -8);
\draw[ultra thin] (9, -18) -- (12, -21);
\draw[ultra thin] (6, -18) -- (6, -21);
\draw[ultra thin] (3, -18) -- (0, -21);
\draw[ultra thin] (3, -15) -- (0, -15);
\draw[ultra thin] (3, -11) -- (0, -11);
\draw[ultra thin] (3, -8) -- (3, -11);
\draw[ultra thin] (5, -8) -- (5, -11);
\draw[ultra thin] (7, -8) -- (7, -11);
\draw[ultra thin] (9, -8) -- (9, -11);
\draw[ultra thin] (9, -11) -- (12, -11);
\draw[ultra thin] (9, -13) -- (12, -13);
\draw[ultra thin] (9, -15) -- (12, -15);
\draw[ultra thin] (9, -17) -- (12, -17);
\draw[thick, red, edge2] (0.5, -0) -- (1.5, -2);
\draw[thick, dash pattern=on 3pt off 3pt , red, edge2] (0.5, -0) -- (0.5, -2);
\draw[thick, red, edge2] (1.5, -0) -- (1.5, -2);
\draw[thick, red, edge2] (0.5, -2) -- (1.5, -4);
\draw[thick, dash pattern=on 3pt off 3pt , red, edge2] (0.5, -2) -- (0.5, -4);
\draw[thick, red, edge2] (1.5, -2) -- (1.5, -4);
\draw[thick, red, edge2] (0.5, -4) -- (1.5, -6);
\draw[thick, dash pattern=on 3pt off 3pt , red, edge2] (0.5, -4) -- (0.5, -6);
\draw[thick, red, edge2] (1.5, -4) -- (1.5, -6);
\draw[thick, red, edge2] (0.5, -6) -- (1.5, -8);
\draw[thick, red, edge2] (1.5, -6) -- (1.5, -8);
\draw[thick, red, edge2] (0.5, -8) -- (0.5, -11);
\draw[thick, red, edge2] (1.5, -8) -- (3, -9.5);
\draw[thick, red, edge2] (9.5, -0) -- (9.5, -2);
\draw[thick, red, edge2] (10.5, -0) -- (9.5, -2);
\draw[thick, dash pattern=on 3pt off 3pt , red, edge2] (10.5, -0) -- (10.5, -2);
\draw[thick, red, edge2] (9.5, -2) -- (9.5, -4);
\draw[thick, red, edge2] (10.5, -2) -- (9.5, -4);
\draw[thick, dash pattern=on 3pt off 3pt , red, edge2] (10.5, -2) -- (10.5, -4);
\draw[thick, red, edge2] (11.5, -2) -- (11.5, -0);
\draw[thick, red, edge2] (9.5, -4) -- (9.5, -6);
\draw[thick, red, edge2] (10.5, -4) -- (9.5, -6);
\draw[thick, dash pattern=on 3pt off 3pt , red, edge2] (10.5, -4) -- (10.5, -6);
\draw[thick, red, edge2] (11.5, -4) -- (11.5, -2);
\draw[thick, red, edge2] (9.5, -6) -- (9.5, -8);
\draw[thick, red, edge2] (10.5, -6) -- (9.5, -8);
\draw[thick, dash pattern=on 3pt off 3pt , red, edge2] (10.5, -6) -- (10.5, -8);
\draw[thick, red, edge2] (11.5, -6) -- (11.5, -4);
\draw[thick, red, edge2] (9.5, -8) -- (9, -8.5);
\draw[thick, red, edge2] (10.5, -8) -- (9.5, -11);
\draw[thick, red, edge2] (11.5, -8) -- (11.5, -6);
\draw[thick, red, edge2] (9.5, -18.5) -- (6, -18.5);
\draw[thick, red, edge2] (10.5, -19.5) -- (6, -18.5);
\draw[thick, dash pattern=on 3pt off 3pt , red, edge2] (10.5, -19.5) -- (6, -19.5);
\draw[thick, red, edge2] (11.5, -20.5) -- (11.5, -17);
\draw[thick, red, edge2] (6, -18.5) -- (2.5, -18.5);
\draw[thick, red, edge2] (6, -19.5) -- (2.5, -18.5);
\draw[thick, dash pattern=on 3pt off 3pt , red, edge2] (6, -19.5) -- (1.5, -19.5);
\draw[thick, red, edge2] (6, -20.5) -- (11.5, -20.5);
\draw[thick, red, edge2] (2.5, -18.5) -- (2.5, -15);
\draw[thick, red, edge2] (1.5, -19.5) -- (2.5, -15);
\draw[thick, dash pattern=on 3pt off 3pt , red, edge2] (1.5, -19.5) -- (1.5, -15);
\draw[thick, red, edge2] (0.5, -20.5) -- (6, -20.5);
\draw[thick, red, edge2] (2.5, -15) -- (2.5, -11);
\draw[thick, red, edge2] (1.5, -15) -- (2.5, -11);
\draw[thick, red, edge2] (0.5, -15) -- (0.5, -20.5);
\draw[thick, red, edge2] (2.5, -11) -- (3, -10.5);
\draw[thick, red, edge2] (0.5, -11) -- (0.5, -15);
\draw[thick, red, edge2] (3, -10.5) -- (5, -9.5);
\draw[thick, dash pattern=on 3pt off 3pt , red, edge2] (3, -10.5) -- (5, -10.5);
\draw[thick, red, edge2] (3, -9.5) -- (5, -9.5);
\draw[thick, red, edge2] (3, -8.5) -- (0.5, -11);
\draw[thick, red, edge2] (5, -10.5) -- (7, -9.5);
\draw[thick, dash pattern=on 3pt off 3pt , red, edge2] (5, -10.5) -- (7, -10.5);
\draw[thick, red, edge2] (5, -9.5) -- (7, -9.5);
\draw[thick, red, edge2] (5, -8.5) -- (3, -8.5);
\draw[thick, red, edge2] (7, -10.5) -- (9, -9.5);
\draw[thick, red, edge2] (7, -9.5) -- (9, -9.5);
\draw[thick, red, edge2] (7, -8.5) -- (5, -8.5);
\draw[thick, red, edge2] (9, -10.5) -- (9.5, -11);
\draw[thick, red, edge2] (9, -9.5) -- (11.5, -8);
\draw[thick, red, edge2] (9, -8.5) -- (7, -8.5);
\draw[thick, red, edge2] (9.5, -11) -- (9.5, -13);
\draw[thick, red, edge2] (9.5, -13) -- (9.5, -15);
\draw[thick, red, edge2] (10.5, -15) -- (9.5, -17);
\draw[thick, dash pattern=on 3pt off 3pt , red, edge2] (10.5, -15) -- (10.5, -17);
\draw[thick, red, edge2] (9.5, -15) -- (9.5, -17);
\draw[thick, red, edge2] (10.5, -17) -- (9.5, -18.5);
\draw[thick, dash pattern=on 3pt off 3pt , red, edge2] (10.5, -17) -- (10.5, -19.5);
\draw[thick, red, edge2] (9.5, -17) -- (9.5, -18.5);
\draw[thick, red, edge2] (11.5, -17) -- (11.5, -15);
\draw[thick, red, edge2] (11.5, -15) -- (11.5, -13);
\draw[thick, red, edge2] (10.5, -13) -- (10.5, -11);
\draw[thick, red, edge2] (11.5, -13) -- (11.5, -11);
\draw[thick, red, edge2] (10.5, -11) -- (9, -8.5);
\draw[thick, red, edge2] (11.5, -11) -- (11.5, -8);
\end{tikzpicture} &
\quad\quad \begin{tikzpicture}[scale=0.29]
\filldraw[red] (11.5, 0) circle (10pt);
\draw[thick] (0, -0) -- (0, -21);
\draw[thick] (0, -21) -- (12, -21);
\draw[thick] (12, -21) -- (12, -0);
\draw[thick] (3, -0) -- (3, -8);
\draw[thick] (3, -8) -- (9, -8);
\draw[thick] (9, -8) -- (9, -0);
\draw[thick] (3, -11) -- (3, -18);
\draw[thick] (3, -18) -- (9, -18);
\draw[thick] (9, -18) -- (9, -11);
\draw[thick] (9, -11) -- (3, -11);
\draw[thick] (2, -0) -- (2, -8);
\draw[thick, dash pattern=on 6pt off 12pt] (1, -0) -- (1, -8);
\draw[thick] (11, -0) -- (11, -8);
\draw[thick, dash pattern=on 6pt off 12pt] (10, -0) -- (10, -8);
\draw[thick] (11, -20) -- (1, -20);
\draw[thick] (1, -20) -- (1, -11);
\draw[thick, dash pattern=on 6pt off 12pt] (10, -19) -- (2, -19);
\draw[thick, dash pattern=on 6pt off 12pt] (2, -19) -- (2, -11);
\draw[thick, dash pattern=on 6pt off 12pt] (3, -10) -- (9, -10);
\draw[thick] (3, -9) -- (9, -9);
\draw[ultra thin] (11, -11) -- (11, -13);
\draw[thick] (10, -11) -- (10, -13);
\draw[thick] (10, -13) -- (11, -15);
\draw[thick] (11, -15) -- (11, -20);
\draw[thick, dash pattern=on 6pt off 12pt] (10, -15) -- (10, -19);
\draw[ultra thin] (0, -0) -- (3, -0);
\draw[ultra thin] (0, -2) -- (3, -2);
\draw[ultra thin] (0, -4) -- (3, -4);
\draw[ultra thin] (0, -6) -- (3, -6);
\draw[ultra thin] (0, -8) -- (3, -8);
\draw[ultra thin] (9, -0) -- (12, -0);
\draw[ultra thin] (9, -2) -- (12, -2);
\draw[ultra thin] (9, -4) -- (12, -4);
\draw[ultra thin] (9, -6) -- (12, -6);
\draw[ultra thin] (9, -8) -- (12, -8);
\draw[ultra thin] (9, -18) -- (12, -21);
\draw[ultra thin] (6, -18) -- (6, -21);
\draw[ultra thin] (3, -18) -- (0, -21);
\draw[ultra thin] (3, -15) -- (0, -15);
\draw[ultra thin] (3, -11) -- (0, -11);
\draw[ultra thin] (3, -8) -- (3, -11);
\draw[ultra thin] (5, -8) -- (5, -11);
\draw[ultra thin] (7, -8) -- (7, -11);
\draw[ultra thin] (9, -8) -- (9, -11);
\draw[ultra thin] (9, -11) -- (12, -11);
\draw[ultra thin] (9, -13) -- (12, -13);
\draw[ultra thin] (9, -15) -- (12, -15);
\draw[ultra thin] (9, -17) -- (12, -17);
\draw[thick, red, edge2] (0.5, -0) -- (0.5, -2);
\draw[thick, red, edge2] (1.5, -0) -- (1.5, -2);
\draw[thick, red, edge2] (0.5, -2) -- (0.5, -4);
\draw[thick, red, edge2] (1.5, -2) -- (1.5, -4);
\draw[thick, red, edge2] (0.5, -4) -- (0.5, -6);
\draw[thick, red, edge2] (1.5, -4) -- (1.5, -6);
\draw[thick, red, edge2] (0.5, -6) -- (0.5, -8);
\draw[thick, red, edge2] (1.5, -6) -- (1.5, -8);
\draw[thick, red, edge2] (0.5, -8) -- (0.5, -11);
\draw[thick, red, edge2] (1.5, -8) -- (3, -9.5);
\draw[thick, red, edge2] (9.5, -0) -- (9.5, -2);
\draw[thick, red, edge2] (10.5, -0) -- (9.5, -2);
\draw[thick, dash pattern=on 3pt off 3pt , red, edge2] (10.5, -0) -- (10.5, -2);
\draw[thick, red, edge2] (9.5, -2) -- (9.5, -4);
\draw[thick, red, edge2] (10.5, -2) -- (9.5, -4);
\draw[thick, dash pattern=on 3pt off 3pt , red, edge2] (10.5, -2) -- (10.5, -4);
\draw[thick, red, edge2] (11.5, -2) -- (11.5, -0);
\draw[thick, red, edge2] (9.5, -4) -- (9.5, -6);
\draw[thick, red, edge2] (10.5, -4) -- (9.5, -6);
\draw[thick, dash pattern=on 3pt off 3pt , red, edge2] (10.5, -4) -- (10.5, -6);
\draw[thick, red, edge2] (11.5, -4) -- (11.5, -2);
\draw[thick, red, edge2] (9.5, -6) -- (9.5, -8);
\draw[thick, red, edge2] (10.5, -6) -- (9.5, -8);
\draw[thick, dash pattern=on 3pt off 3pt , red, edge2] (10.5, -6) -- (10.5, -8);
\draw[thick, red, edge2] (11.5, -6) -- (11.5, -4);
\draw[thick, red, edge2] (9.5, -8) -- (9, -8.5);
\draw[thick, red, edge2] (10.5, -8) -- (9.5, -11);
\draw[thick, red, edge2] (11.5, -8) -- (11.5, -6);
\draw[thick, red, edge2] (9.5, -18.5) -- (6, -18.5);
\draw[thick, red, edge2] (10.5, -19.5) -- (6, -18.5);
\draw[thick, dash pattern=on 3pt off 3pt , red, edge2] (10.5, -19.5) -- (6, -19.5);
\draw[thick, red, edge2] (11.5, -20.5) -- (11.5, -17);
\draw[thick, red, edge2] (6, -18.5) -- (2.5, -18.5);
\draw[thick, red, edge2] (6, -19.5) -- (2.5, -18.5);
\draw[thick, dash pattern=on 3pt off 3pt , red, edge2] (6, -19.5) -- (1.5, -19.5);
\draw[thick, red, edge2] (6, -20.5) -- (11.5, -20.5);
\draw[thick, red, edge2] (2.5, -18.5) -- (2.5, -15);
\draw[thick, red, edge2] (1.5, -19.5) -- (2.5, -15);
\draw[thick, dash pattern=on 3pt off 3pt , red, edge2] (1.5, -19.5) -- (1.5, -15);
\draw[thick, red, edge2] (0.5, -20.5) -- (6, -20.5);
\draw[thick, red, edge2] (2.5, -15) -- (2.5, -11);
\draw[thick, red, edge2] (1.5, -15) -- (2.5, -11);
\draw[thick, red, edge2] (0.5, -15) -- (0.5, -20.5);
\draw[thick, red, edge2] (2.5, -11) -- (3, -10.5);
\draw[thick, red, edge2] (0.5, -11) -- (0.5, -15);
\draw[thick, red, edge2] (3, -10.5) -- (5, -9.5);
\draw[thick, dash pattern=on 3pt off 3pt , red, edge2] (3, -10.5) -- (5, -10.5);
\draw[thick, red, edge2] (3, -9.5) -- (5, -9.5);
\draw[thick, red, edge2] (3, -8.5) -- (0.5, -11);
\draw[thick, red, edge2] (5, -10.5) -- (7, -9.5);
\draw[thick, dash pattern=on 3pt off 3pt , red, edge2] (5, -10.5) -- (7, -10.5);
\draw[thick, red, edge2] (5, -9.5) -- (7, -9.5);
\draw[thick, red, edge2] (5, -8.5) -- (3, -8.5);
\draw[thick, red, edge2] (7, -10.5) -- (9, -9.5);
\draw[thick, red, edge2] (7, -9.5) -- (9, -9.5);
\draw[thick, red, edge2] (7, -8.5) -- (5, -8.5);
\draw[thick, red, edge2] (9, -10.5) -- (9.5, -11);
\draw[thick, red, edge2] (9, -9.5) -- (11.5, -8);
\draw[thick, red, edge2] (9, -8.5) -- (7, -8.5);
\draw[thick, red, edge2] (9.5, -11) -- (9.5, -13);
\draw[thick, red, edge2] (9.5, -13) -- (9.5, -15);
\draw[thick, red, edge2] (10.5, -15) -- (9.5, -17);
\draw[thick, dash pattern=on 3pt off 3pt , red, edge2] (10.5, -15) -- (10.5, -17);
\draw[thick, red, edge2] (9.5, -15) -- (9.5, -17);
\draw[thick, red, edge2] (10.5, -17) -- (9.5, -18.5);
\draw[thick, dash pattern=on 3pt off 3pt , red, edge2] (10.5, -17) -- (10.5, -19.5);
\draw[thick, red, edge2] (9.5, -17) -- (9.5, -18.5);
\draw[thick, red, edge2] (11.5, -17) -- (11.5, -15);
\draw[thick, red, edge2] (11.5, -15) -- (11.5, -13);
\draw[thick, red, edge2] (10.5, -13) -- (10.5, -11);
\draw[thick, red, edge2] (11.5, -13) -- (11.5, -11);
\draw[thick, red, edge2] (10.5, -11) -- (9, -8.5);
\draw[thick, red, edge2] (11.5, -11) -- (11.5, -8);
\end{tikzpicture} &
\quad\quad \begin{tikzpicture}[scale=0.29]
\filldraw[red] (11.5, 0) circle (10pt);
\draw[thick] (0, -0) -- (0, -21);
\draw[thick] (0, -21) -- (12, -21);
\draw[thick] (12, -21) -- (12, -0);
\draw[thick] (3, -0) -- (3, -8);
\draw[thick] (3, -8) -- (9, -8);
\draw[thick] (9, -8) -- (9, -0);
\draw[thick] (3, -11) -- (3, -18);
\draw[thick] (3, -18) -- (9, -18);
\draw[thick] (9, -18) -- (9, -11);
\draw[thick] (9, -11) -- (3, -11);
\draw[thick] (2, -0) -- (2, -8);
\draw[thick, dash pattern=on 6pt off 12pt] (1, -0) -- (1, -8);
\draw[thick] (11, -0) -- (11, -8);
\draw[thick, dash pattern=on 6pt off 12pt] (10, -0) -- (10, -8);
\draw[thick] (11, -20) -- (1, -20);
\draw[thick] (1, -20) -- (1, -11);
\draw[thick, dash pattern=on 6pt off 12pt] (10, -19) -- (2, -19);
\draw[thick, dash pattern=on 6pt off 12pt] (2, -19) -- (2, -11);
\draw[thick, dash pattern=on 6pt off 12pt] (3, -10) -- (9, -10);
\draw[thick] (3, -9) -- (9, -9);
\draw[ultra thin] (11, -11) -- (11, -13);
\draw[thick] (10, -11) -- (10, -13);
\draw[thick] (10, -13) -- (11, -15);
\draw[thick] (11, -15) -- (11, -20);
\draw[thick, dash pattern=on 6pt off 12pt] (10, -15) -- (10, -19);
\draw[ultra thin] (0, -0) -- (3, -0);
\draw[ultra thin] (0, -2) -- (3, -2);
\draw[ultra thin] (0, -4) -- (3, -4);
\draw[ultra thin] (0, -6) -- (3, -6);
\draw[ultra thin] (0, -8) -- (3, -8);
\draw[ultra thin] (9, -0) -- (12, -0);
\draw[ultra thin] (9, -2) -- (12, -2);
\draw[ultra thin] (9, -4) -- (12, -4);
\draw[ultra thin] (9, -6) -- (12, -6);
\draw[ultra thin] (9, -8) -- (12, -8);
\draw[ultra thin] (9, -18) -- (12, -21);
\draw[ultra thin] (6, -18) -- (6, -21);
\draw[ultra thin] (3, -18) -- (0, -21);
\draw[ultra thin] (3, -15) -- (0, -15);
\draw[ultra thin] (3, -11) -- (0, -11);
\draw[ultra thin] (3, -8) -- (3, -11);
\draw[ultra thin] (5, -8) -- (5, -11);
\draw[ultra thin] (7, -8) -- (7, -11);
\draw[ultra thin] (9, -8) -- (9, -11);
\draw[ultra thin] (9, -11) -- (12, -11);
\draw[ultra thin] (9, -13) -- (12, -13);
\draw[ultra thin] (9, -15) -- (12, -15);
\draw[ultra thin] (9, -17) -- (12, -17);
\draw[thick, red, edge2] (0.5, -0) -- (1.5, -2);
\draw[thick, dash pattern=on 3pt off 3pt , red, edge2] (0.5, -0) -- (0.5, -2);
\draw[thick, red, edge2] (1.5, -0) -- (1.5, -2);
\draw[thick, red, edge2] (0.5, -2) -- (1.5, -4);
\draw[thick, dash pattern=on 3pt off 3pt , red, edge2] (0.5, -2) -- (0.5, -4);
\draw[thick, red, edge2] (1.5, -2) -- (1.5, -4);
\draw[thick, red, edge2] (0.5, -4) -- (1.5, -6);
\draw[thick, dash pattern=on 3pt off 3pt , red, edge2] (0.5, -4) -- (0.5, -6);
\draw[thick, red, edge2] (1.5, -4) -- (1.5, -6);
\draw[thick, red, edge2] (0.5, -6) -- (1.5, -8);
\draw[thick, dash pattern=on 3pt off 3pt , red, edge2] (0.5, -6) -- (0.5, -8);
\draw[thick, red, edge2] (1.5, -6) -- (1.5, -8);
\draw[thick, red, edge2] (0.5, -8) -- (0.5, -11);
\draw[thick, red, edge2] (1.5, -8) -- (3, -9.5);
\draw[thick, red, edge2] (9.5, -0) -- (9.5, -2);
\draw[thick, red, edge2] (10.5, -0) -- (9.5, -2);
\draw[thick, dash pattern=on 3pt off 3pt , red, edge2] (10.5, -0) -- (10.5, -2);
\draw[thick, red, edge2] (9.5, -2) -- (9.5, -4);
\draw[thick, red, edge2] (10.5, -2) -- (9.5, -4);
\draw[thick, dash pattern=on 3pt off 3pt , red, edge2] (10.5, -2) -- (10.5, -4);
\draw[thick, red, edge2] (11.5, -2) -- (11.5, -0);
\draw[thick, red, edge2] (9.5, -4) -- (9.5, -6);
\draw[thick, red, edge2] (10.5, -4) -- (9.5, -6);
\draw[thick, dash pattern=on 3pt off 3pt , red, edge2] (10.5, -4) -- (10.5, -6);
\draw[thick, red, edge2] (11.5, -4) -- (11.5, -2);
\draw[thick, red, edge2] (9.5, -6) -- (9.5, -8);
\draw[thick, red, edge2] (10.5, -6) -- (9.5, -8);
\draw[thick, dash pattern=on 3pt off 3pt , red, edge2] (10.5, -6) -- (10.5, -8);
\draw[thick, red, edge2] (11.5, -6) -- (11.5, -4);
\draw[thick, red, edge2] (9.5, -8) -- (9, -8.5);
\draw[thick, red, edge2] (10.5, -8) -- (9.5, -11);
\draw[thick, red, edge2] (11.5, -8) -- (11.5, -6);
\draw[thick, red, edge2] (9.5, -18.5) -- (6, -18.5);
\draw[thick, red, edge2] (10.5, -19.5) -- (6, -18.5);
\draw[thick, dash pattern=on 3pt off 3pt , red, edge2] (10.5, -19.5) -- (6, -19.5);
\draw[thick, red, edge2] (11.5, -20.5) -- (11.5, -17);
\draw[thick, red, edge2] (6, -18.5) -- (2.5, -18.5);
\draw[thick, red, edge2] (6, -19.5) -- (2.5, -18.5);
\draw[thick, dash pattern=on 3pt off 3pt , red, edge2] (6, -19.5) -- (1.5, -19.5);
\draw[thick, red, edge2] (6, -20.5) -- (11.5, -20.5);
\draw[thick, red, edge2] (2.5, -18.5) -- (2.5, -15);
\draw[thick, red, edge2] (1.5, -19.5) -- (2.5, -15);
\draw[thick, dash pattern=on 3pt off 3pt , red, edge2] (1.5, -19.5) -- (1.5, -15);
\draw[thick, red, edge2] (0.5, -20.5) -- (6, -20.5);
\draw[thick, red, edge2] (2.5, -15) -- (2.5, -11);
\draw[thick, red, edge2] (1.5, -15) -- (2.5, -11);
\draw[thick, red, edge2] (0.5, -15) -- (0.5, -20.5);
\draw[thick, red, edge2] (2.5, -11) -- (3, -10.5);
\draw[thick, red, edge2] (0.5, -11) -- (0.5, -15);
\draw[thick, red, edge2] (3, -10.5) -- (5, -9.5);
\draw[thick, dash pattern=on 3pt off 3pt , red, edge2] (3, -10.5) -- (5, -10.5);
\draw[thick, red, edge2] (3, -9.5) -- (5, -9.5);
\draw[thick, red, edge2] (3, -8.5) -- (0.5, -11);
\draw[thick, red, edge2] (5, -10.5) -- (7, -9.5);
\draw[thick, dash pattern=on 3pt off 3pt , red, edge2] (5, -10.5) -- (7, -10.5);
\draw[thick, red, edge2] (5, -9.5) -- (7, -9.5);
\draw[thick, red, edge2] (5, -8.5) -- (3, -8.5);
\draw[thick, red, edge2] (7, -10.5) -- (9, -9.5);
\draw[thick, dash pattern=on 3pt off 3pt , red, edge2] (7, -10.5) -- (9, -10.5);
\draw[thick, red, edge2] (7, -9.5) -- (9, -9.5);
\draw[thick, red, edge2] (7, -8.5) -- (5, -8.5);
\draw[thick, red, edge2] (9, -10.5) -- (9.5, -11);
\draw[thick, red, edge2] (9, -9.5) -- (11.5, -8);
\draw[thick, red, edge2] (9, -8.5) -- (7, -8.5);
\draw[thick, red, edge2] (9.5, -11) -- (9.5, -13);
\draw[thick, red, edge2] (9.5, -13) -- (9.5, -15);
\draw[thick, red, edge2] (10.5, -15) -- (9.5, -17);
\draw[thick, dash pattern=on 3pt off 3pt , red, edge2] (10.5, -15) -- (10.5, -17);
\draw[thick, red, edge2] (9.5, -15) -- (9.5, -17);
\draw[thick, red, edge2] (10.5, -17) -- (9.5, -18.5);
\draw[thick, dash pattern=on 3pt off 3pt , red, edge2] (10.5, -17) -- (10.5, -19.5);
\draw[thick, red, edge2] (9.5, -17) -- (9.5, -18.5);
\draw[thick, red, edge2] (11.5, -17) -- (11.5, -15);
\draw[thick, red, edge2] (11.5, -15) -- (11.5, -13);
\draw[thick, red, edge2] (10.5, -13) -- (10.5, -11);
\draw[thick, red, edge2] (11.5, -13) -- (11.5, -11);
\draw[thick, red, edge2] (10.5, -11) -- (9, -8.5);
\draw[thick, red, edge2] (11.5, -11) -- (11.5, -8);
\end{tikzpicture} \\
{\footnotesize (a) low $c_\lc$} &
{\footnotesize \quad\quad (b) high $c_\lc$} &
{\footnotesize \quad\quad (c) $\alpha c_\flc > c(x) / \ell(x)$ } \\
\end{tabular}
\caption{Running the value iteration algorithm over the same network as presented in \figref{intersections} for two different values of lane change cost, as well as a forced lane change cost that violates the monotonicity requirement. The goal is in the top-right corner.}
\label{fig:intersections-vi}
\end{figure*}

Given that the resulting policies look similar on the mock network, we now turn to the computational efficiency of our algorithm. When comparing the running time of our algorithm against existing algorithms provided by the AI-toolbox library, we represented the MDP as a sparse matrix to obtain the best possible performance out of the library. We ran both the policy and value iteration algorithms until the policies converged to the default tolerance of 0.001. A discount factor of 0.9 was used. All algorithms were run on a variety of real road networks mapped by Nuro, as seen in \tabref{comparisons}.

\begin{table*}
\resizebox{\textwidth}{!}{%
\begin{tabular}{l|rrrrrr}
\textbf{Map} & \multicolumn{1}{l}{$\boldsymbol{n}$} & \multicolumn{1}{l}{$\boldsymbol{m}$} & \multicolumn{1}{l}{\textbf{Avg. Degree}} & \multicolumn{1}{l}{\textbf{Algorithm~\ref{alg:router} (ms)}} & \multicolumn{1}{l}{\textbf{Value iteration (ms)}} & \multicolumn{1}{l}{\textbf{Policy iteration (ms)}} \\ \hline
Las Vegas Motor Speedway, NV & 26564 & 43599 & 3.28256 & 5.855 & 14.5862 & 65.5226\\
Phoenix, AZ & 72709 & 118008 & 3.24604 & 38.7407 & 107.696 & 398.494\\
San Francisco, CA & 258436 & 409382 & 3.16815 & 111.938 & 320.786 & 1454.06\\
Houston, TX & 332181 & 542484 & 3.2662 & 165.378 & 649.067 & 2003.37\\
Mountain View, CA & 666292 & 1052020 & 3.15784 & 376.413 & 1683.61 & 5484.36
\end{tabular}
}
\vspace{5pt}%
\caption{Comparing the running time (in milliseconds) of our algorithm against existing solutions for solving MDPs. Here, $n$ is the number of vertices in the lane graph and $m$ is the number of edges. Our algorithm outperforms existing solutions significantly.}
\tablab{comparisons}
\end{table*}

\tabref{comparisons} reveals that Algorithm~\ref{alg:router} outperforms existing open-source implementations significantly, as it is roughly 3--4 times faster than the value iteration implementation. This is the main advantage our approach offers over existing solutions, as it is able to generate routes across entire cities (e.g. Mountain View, CA) in less than half of a second with over half of a million lane graph cells. \figref{hou} and \figref{mtv} illustrate another execution of Algorithm~\ref{alg:router} over the Houston, TX and Mountain, CA maps respectively, showing the set of cells explored by the algorithm at different time steps.

\begin{figure*}[h]
\centering
\includegraphics[scale=0.4]{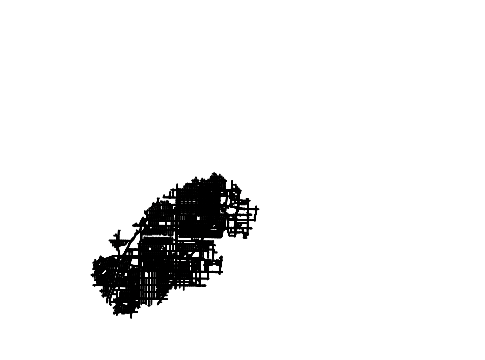}
\quad
\includegraphics[scale=0.4]{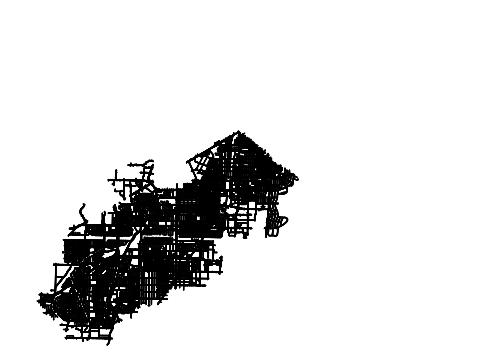}
\quad
\includegraphics[scale=0.4]{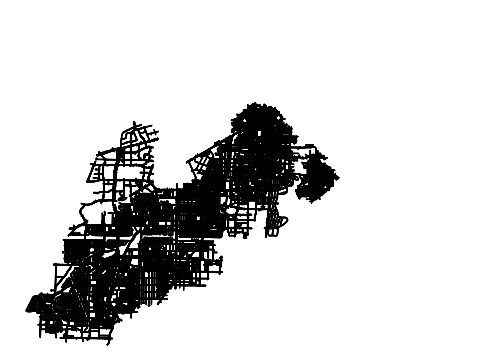}
\quad
\includegraphics[scale=0.4]{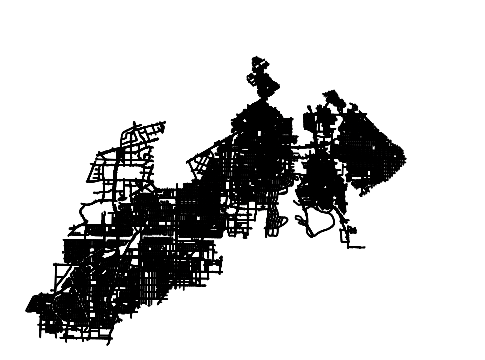}
\quad
\includegraphics[scale=0.4]{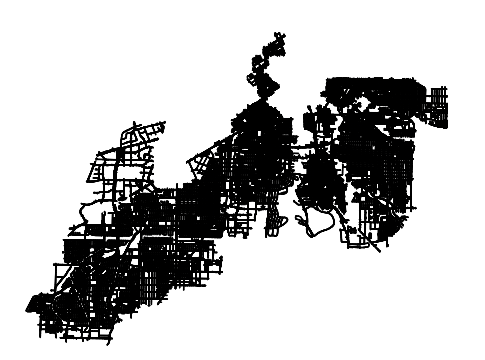}
\caption{Running the algorithm of \thmref{mdp-dijkstra} over a map of Houston, TX. Each figure shows the collection of closed cells at a given point in time during the execution of the algorithm. The entire process took approximately 173 milliseconds, and thus each frame shows the set of newly closed cells approximately every 28 milliseconds.}
\label{fig:hou}
\end{figure*}

\begin{figure*}[h]
\centering
\includegraphics[scale=0.4]{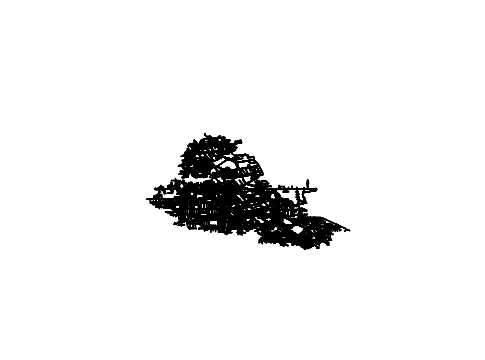}
\quad
\includegraphics[scale=0.4]{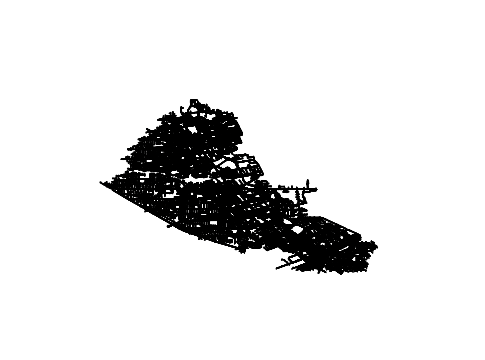}
\quad
\includegraphics[scale=0.4]{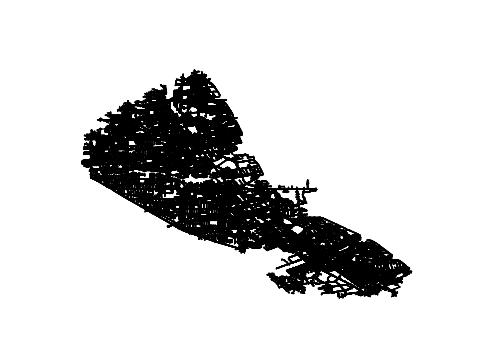}
\quad
\includegraphics[scale=0.4]{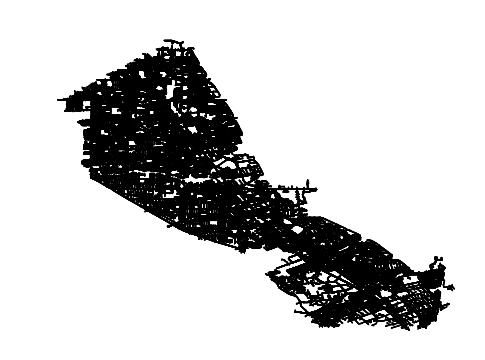}
\quad
\includegraphics[scale=0.4]{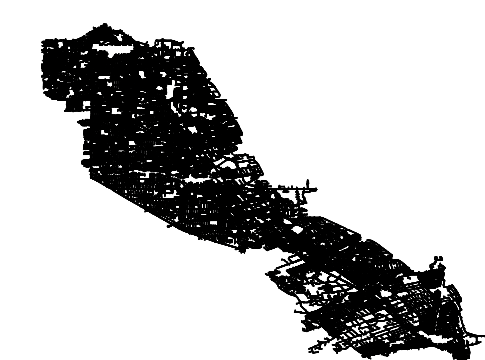}
\caption{Running the algorithm of \thmref{mdp-dijkstra} over a map of Mountain View, CA. Each figure shows the collection of closed cells at a given point in time during the execution of the algorithm. The entire process took approximately 401 milliseconds, and thus each frame shows the set of newly closed cells approximately every 66 milliseconds.}
\label{fig:mtv}
\end{figure*}

\section{Discussion}
\seclab{conclusion}
We presented an approach to provide lane-level routing for autonomous vehicles. We modeled the problem as an Markov Decision Process, while proving that we can still solve it using an efficient Dijkstra-like algorithm. The algorithm allows autonomous vehicles, including Nuro's delivery robots, to make appropriate choices regarding what lane to drive in and when to change lanes.

A topic of ongoing research is how to extend the algorithm to make use of real-time traffic data. While in the experiments we let the cost of cells equal their lengths, we could instead use their travel time if an average speed is available for each cell. In that case we can set $c_\flc = 1/(\alpha v_\mathrm{max})$ to ensure that the monotonicity condition holds. It should also be noted that all proofs regarding the monotonicity of our problem are local in nature. This means that $\alpha$, the lane change success rate, can be locally varied without sacrificing monotonicity (as long as $\alpha$ and $c_\flc$ remain consistent locally). This can be used to reduce the probability of lane changes succeeding (and increase the cost of forced lane changes) in dense traffic conditions, which would cause our algorithm to adjust its behavior and perform necessary lane changes earlier in such conditions.

We presented a consistent heuristic that can focus the search when the minimum cell length is sufficiently large. However, it remains open to obtain a useful consistent heuristic when the cell length is small. We are continuing to explore ways to apply heuristics in some form. One idea is to use a hierarchical approach with different levels of granularity in each hierarchy. Another is to use landmarking-based heuristics \citep{gh-cspsmgt-05} where in a pre-processing step one computes the optimal cost from every cell to every landmark, and uses those distances to inform the search. Another potential avenue for reducing the computation time is to use branch-and-bound techniques such as, e.g., domain restriction \citep{ccv-cdree-14}.

\section*{Acknowledgements}
The authors thank the anonymous referees for their detailed comments and review.

\bibliographystyle{plain}
\bibliography{refs.bib}

\end{document}